\def\DiMOSArxiv{1}
\documentclass{article}

\PassOptionsToPackage{table}{xcolor}
\usepackage[T1]{fontenc}
\ifdefined\DiMOSArxiv
  \usepackage{arxiv,times}
  \iclrfinalcopy
\else
  \usepackage{iclr2027_conference,times}
\fi

\usepackage{dsfont} % for \mathbb{1} as \mathds{1}

\usepackage{amsmath,amsfonts,bm}

\def\1{\bm{1}}

\DeclareMathAlphabet{\mathsfit}{\encodingdefault}{\sfdefault}{m}{sl}
\SetMathAlphabet{\mathsfit}{bold}{\encodingdefault}{\sfdefault}{bx}{n}

\usepackage{amsmath}
\usepackage{amssymb}
\usepackage{amsfonts}
\usepackage{amsthm}
\usepackage{booktabs}
\usepackage{multirow}
\usepackage{makecell}
\usepackage{graphicx}
\usepackage{subcaption}
\usepackage{wrapfig}
\usepackage{algorithm}
\usepackage{algpseudocode}
\usepackage[table]{xcolor}
\usepackage{xspace}
\usepackage{pifont}
\usepackage{enumitem}
\usepackage{microtype}
\usepackage{url}
\usepackage{hyperref}
\usepackage{cleveref}

\newtheorem{theorem}{Theorem}

\newtheorem{proposition}{Proposition}

\newtheorem{definition}{Definition}

\definecolor{dimosAccent}{HTML}{9B2C4E}
\newcommand{\methodname}{\textcolor{dimosAccent}{\textsc{DiMOS}}\xspace}

\definecolor{dimosTask}{HTML}{DDEBF7}
\definecolor{dimosOurs}{HTML}{F8E1E1}
\definecolor{dimosNavy}{HTML}{24527A}
\definecolor{dimosStd}{HTML}{707070}
\definecolor{dimosRule}{HTML}{58606A}
\newcolumntype{I}{!{\color{black}\vrule width 0.4pt}}
\newcolumntype{J}{!{\color{black}\vrule width 0.8pt}}
\newcommand{\dimosOuterRule}{\Xhline{1.2pt}}
\newcommand{\dimosMajorRule}{\Xhline{1pt}}
\newcommand{\dimosSectionRule}{\Xhline{0.6pt}}
\newcommand{\dimosMinorRule}{\Xhline{0.4pt}}

\newcommand{\meanstd}[2]{#1{\scriptsize\color{dimosStd}$\pm$#2}}
\newcommand{\firstresult}[1]{\textbf{#1}}
\newcommand{\secondresult}[1]{\underline{#1}}
\newcommand{\baseline}[1]{\texttt{#1}}

\title{\methodname: Doob-Guided Inference-Time Multi-Objective Search for Scientific Design}
\ifdefined\DiMOSArxiv
  \author{%
    Ziqing Wang\thanks{Equal contribution.}\quad
    Qijie Zhu\footnotemark[1]\quad Weimin Wu\footnotemark[1]\quad Zeqi Ye\\
    \textbf{Minshuo Chen\quad Han Liu\thanks{Corresponding authors.}\quad
    Kaize Ding\footnotemark[2]}\\[4pt]
    \normalfont Northwestern University
  }
  \hypersetup{
    pdftitle={DiMOS: Doob-Guided Inference-Time Multi-Objective Search for Scientific Design},
    pdfauthor={Ziqing Wang, Qijie Zhu, Weimin Wu, Zeqi Ye, Minshuo Chen, Han Liu, Kaize Ding}
  }
\else
  \author{Anonymous authors}
  \urldef\codeurl\url{https://anonymous.4open.science/r/DiMOS/}
\fi

\begin{document}

\maketitle
\suppressfloats[t]

\begin{abstract}
Scientific design often requires jointly satisfying multiple objectives and
constraints. Pretrained masked diffusion models provide a generative
foundation for this task, but fine-tuning them to meet these objectives and
constraints incurs additional training costs, motivating inference-time
guidance with frozen models. However, such guidance faces two challenges:
pass-or-fail constraints and black-box reward models may provide no useful
gradients, while jointly satisfying multiple requirements can leave a small
feasible region, making feasible designs difficult to find within a limited
inference budget. To address these challenges, we introduce \methodname,
a training-free framework for multi-objective scientific design. Using
joint rewards from candidate completions, \methodname performs approximate
Doob-guided local resampling without requiring reward gradients. To allocate
computation efficiently, it uses budget-efficient trajectory search to focus
computation on promising continuations. Across six DNA, protein, and RNA
tasks, \methodname attains the highest joint success rate at comparable
generation times, up to $1.98\times$ the strongest baseline on DNA and
protein, while maintaining high sequence uniqueness and naturalness.
\ifdefined\DiMOSArxiv\else
Our code is available at \codeurl.
\fi
\end{abstract}

\begin{figure}[t]
  \centering
  \captionsetup{skip=3pt}
  \includegraphics[width=\textwidth]{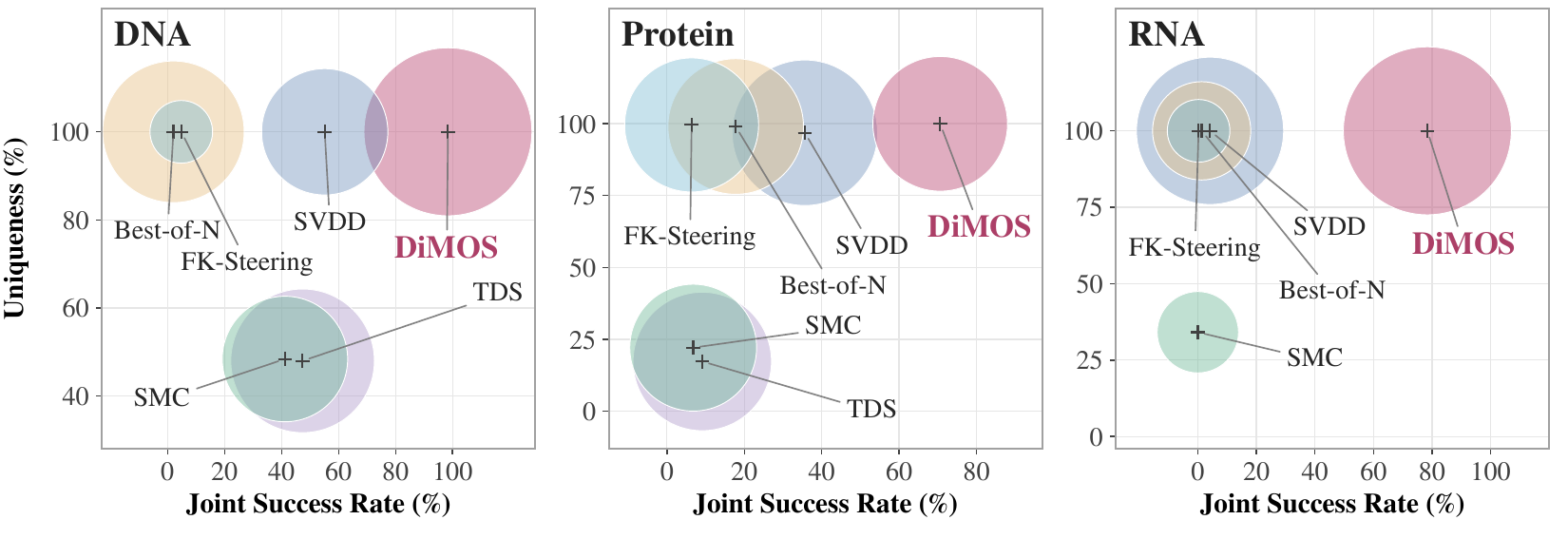}
  \caption{\textbf{Joint success and uniqueness.}
  Bubble area encodes within-task normalized 3-mer correlation (DNA),
  scRMSD $<2\,\text{\AA}$ rate (protein), and Optimus-MRL (RNA).
  Three-seed means.}
  \label{fig:dimos-performance}
  \par\vspace{-12pt}
\end{figure}

\section{Introduction}
\label{sec:introduction}

Scientific sequence design is a central task in synthetic biology and
therapeutic development~\citep{zrimec2022controlling,cao2022design}. In
practice, however, scientists often need to satisfy several requirements
simultaneously rather than improve a single property in
isolation~\citep{stanton2022accelerating,gruver2023protein}. For
example, regulatory DNA should combine target-cell activity with limited off-target
activity~\citep{gosai2023machine}, while therapeutic mRNA requires both
efficient translation and sufficient
stability~\citep{leppek2022combinatorial}. We therefore focus on
\emph{multi-objective scientific sequence design}, seeking candidate
sequences that jointly meet multiple objectives and constraints.

Masked diffusion models provide a generative foundation for this task by
learning sequence distributions through the reconstruction of masked
tokens~\citep{sahoo2024simple,wang2024finetuning}. However, sampling from a
learned distribution does not ensure that designs meet the specified
objectives~\citep{li2024derivative}. Fine-tuning-based methods address this gap
by updating pretrained models toward task-specific
rewards~\citep{black2024training,clark2024directly,wang2024finetuning,han2026discrete,domingoenrich2025adjoint},
but the added training becomes a burden as objectives or thresholds
change~\citep{li2024derivative}. This motivates
inference-time guidance that reuses a frozen generative prior while
steering generation toward different combinations of objectives and
constraints. Such guidance needs one signal to follow, so we assign each
candidate a joint reward $R(x)$ that is high only when every requirement is
met.

Yet guiding a frozen model with such a reward presents two challenges:
\begin{itemize}[leftmargin=*,topsep=1pt,itemsep=4pt,parsep=0pt]
\item \textbf{\ding{182} \textit{How can we guide generation when $R$ is
non-differentiable?}}
Scientific design often requires improving one property while enforcing hard
constraints on others. Even when reward models are differentiable,
pass-or-fail constraints provide no useful gradient signal for correcting
constraint violations. The reward model itself can also be
non-differentiable: black-box tools such as ViennaRNA assess RNA folding
without directly providing sequence gradients~\citep{lorenz2011viennarna}.
Gradient-based inference-time methods therefore struggle to follow such a
reward~\citep{wu2024practical,nisonoff2024unlocking}.
\item \textbf{\ding{183} \textit{How can we effectively search for jointly
feasible designs within a given inference budget?}}
Gradient-free methods handle non-differentiable feedback but often require
extensive sampling and repeated reward
evaluations~\citep{li2024derivative,singhal2025general}. Meanwhile,
simultaneously satisfying multiple objectives and constraints can leave
only a small feasible region, particularly when objectives conflict,
further increasing the search effort required~\citep{stanton2022accelerating}.
Efficiently allocating computation to discover jointly feasible candidates
within a limited inference budget is therefore another key challenge.
\end{itemize}

To address these challenges, we propose \methodname, a training-free
framework for multi-objective scientific design with frozen masked diffusion
models (Fig.~\ref{fig:dimos-overview}). Its approximate Doob-guided local
resampling steers generation using reward evaluations alone, accommodating
non-differentiable feedback. Building on this local guidance, budget-efficient
trajectory search directs computation toward promising continuations to find
jointly feasible designs.
Our three contributions are:
\begin{itemize}[leftmargin=*,topsep=0pt,partopsep=0pt,itemsep=3pt,parsep=0pt]
\item \textbf{Framework.}
We propose \methodname, a Doob-guided framework for inference-time
multi-objective design with frozen masked diffusion models. It supports
black-box objectives and hard constraints through evaluations of $R$,
keeping the generative model frozen and requiring reward values alone.
\item \textbf{Algorithm.}
We develop a budget-efficient trajectory search algorithm that combines
stage-wise branching with reward-guided refinement. It allocates computation
across local candidate comparisons and promising continuations to improve
joint success within a given inference budget.
\item \textbf{Effectiveness.}
Across six multi-objective tasks spanning DNA, protein, and RNA design
(Fig.~\ref{fig:dimos-performance}), \methodname attains the highest joint
success rate under matched wall-clock budgets, up to $1.98\times$ the
strongest baseline on DNA and protein, while maintaining sequence
naturalness.
\end{itemize}

\section{Related Work}
\label{sec:related-work}

\paragraph{Multi-Objective Scientific Design.}
Multi-objective scientific design seeks sequences with several desired
properties. Discrete diffusion supports regulatory DNA
design~\citep{sarkar2024designing}. Constrained reinforcement learning adapts generative models
to optimize target-cell activity while controlling off-target effects
\citep{chen2025ctrldna}. Other approaches explore Pareto trade-offs through
multi-objective Bayesian optimization \citep{stanton2022accelerating},
peptide tree search \citep{tang2025peptune}, or rank-directional flow
guidance \citep{chen2025mogdfm}. Pareto optimality alone, however, does not
ensure simultaneous satisfaction of application thresholds.
Our focus is finding designs that jointly meet specified property
thresholds through inference-time search.

\paragraph{Training-Free Adaptation.}
Training-free methods steer a frozen generative prior at sampling time
using reward feedback. (1) \emph{Gradient guidance} modifies reverse
transitions using reward gradients or gradient-based approximations
\citep{nisonoff2024unlocking,schiff2025simple}.
(2) \emph{Particle methods} reweight and resample trajectories, sometimes
using gradient-guided proposals
\citep{wu2024practical,singhal2025general}.
(3) \emph{Evaluation-based guidance} ranks or resamples candidate
continuations using reward values or estimated future rewards
\citep{guo2024plug,li2024derivative}, with iterative refinement wrapping
single-shot guidance in an outer noising-and-denoising
loop~\citep{uehara2025reward}. That outer loop is orthogonal to the per-step
mechanisms studied here and could wrap any of them.
Related probabilistic formulations connect adaptation to Doob's
$h$-transform, soft value functions, and posterior
sampling~\citep{uehara2025inference}. These include learned corrections
or samplers \citep{denker2024deft,rectorbrooks2025steering}, as well as
training-free inference for image editing \citep{nguyen2025hedit}
and adaptation with generic rewards \citep{zhu2026training}.
\methodname uses reward evaluations alone for approximate Doob-guided
local resampling, then concentrates computation on promising
continuations via shared prefixes and branch selection.

\section{Preliminaries}
\label{sec:preliminaries}

We introduce discrete diffusion \citep{austin2021structured} in Sec. \ref{sec:discrete_diffusion} and Doob's
$h$-transform in Sec. \ref{sec:doob_transform}.

% \citep{austin2021structured} , specialize its
% sampling dynamics to masked discrete diffusion, and describe Doob's
% $h$-transform for adapting the terminal distribution of a frozen generative
% model.

\subsection{Discrete Diffusion}
\label{sec:discrete_diffusion}

Let $\mathcal{X}=\mathcal{V}^N$ be the space of length-$N$ sequences over a
finite alphabet $\mathcal{V}$. A forward continuous-time Markov chain
$(Y_t)_{t\in[0,T]}$ corrupts
$Y_0\sim p_{\mathrm{data}}$ on a finite state space
$\mathcal{S}\supseteq\mathcal{X}$.
We use the column convention: for $y\neq x$,
$Q_t(y,x)\geq 0$ is the jump rate from $x$ to $y$.
The marginal probability vector $p_t$ evolves as
\begin{equation}
    \frac{dp_t}{dt}
    =
    Q_t p_t,
    \qquad
    Q_t(x,x)
    =
    -\sum_{y\neq x}Q_t(y,x),
    \qquad
    p_0=p_{\mathrm{data}}.
    \label{eq:forward_ctmc}
\end{equation}
The corruption process is chosen so that $p_T$ approaches a tractable
reference distribution.

Writing denoising time as $u=T-t$, define the reverse-time process
$\widetilde Y_u := Y_{T-u}$, so that denoising runs from $u=0$ to $u=T$,
corresponding to the forward-time direction $t=T$ to $t=0$.
Its exact off-diagonal generator is
\begin{equation}
    \widetilde Q_u(y,x)
    =
    Q_{T-u}(x,y)
    \frac{p_{T-u}(y)}{p_{T-u}(x)},
    \qquad y\neq x,
    \label{eq:reverse_ctmc}
\end{equation}
for states with $p_{T-u}(x)>0$.
Thus, starting from $\widetilde Y_0\sim p_T$, the reverse process
progressively removes the corruption introduced by the forward process and
recovers $\widetilde Y_T\sim p_0=p_{\mathrm{data}}$.

The reverse dynamics in Eq.~(\ref{eq:reverse_ctmc}) depend on the density
ratios $p_t(y)/p_t(x)$, which characterize how the reverse process should
move between discrete states \citep{lou2024discrete}. 

\paragraph{Masked Discrete Diffusion.}
A commonly used specialization of discrete diffusion for sequence generation
is a masked discrete diffusion model
\citep{sahoo2024simple,shi2024simplified}. For sequence generation, we augment the token alphabet
$\mathcal V$ with a mask symbol $\mathtt m\notin\mathcal V$ and let
$\widetilde{\mathcal V}=\mathcal V\cup\{\mathtt m\}$ and
$\mathcal S=\widetilde{\mathcal V}^N$.
Masked discrete diffusion independently replaces each clean token by
$\mathtt m$ at a time-dependent rate $\lambda_t$
\citep{sahoo2024simple,shi2024simplified}.
Writing
$\alpha_t=\exp(-\int_0^t \lambda_r\,dr)$
for the probability that a coordinate remains unmasked by time $t$, the
forward transition from $x_0\in\mathcal X$ is
\begin{equation}
    q_{0,t}(x\mid x_0)
    =
    \prod_{j=1}^N
    \left[
        \alpha_t \mathbf 1\{x^j=x_0^j\}
        +
        (1-\alpha_t)\mathbf 1\{x^j=\mathtt m\}
    \right].
    \label{eq:masked_corruption}
\end{equation}
Thus, the forward process progressively replaces clean tokens by masks and
never changes one clean token directly into another. As $\alpha_T\to 0$,
$p_T$ approaches this fully masked reference distribution $\delta_{\mathtt m^N}$.

For a masked coordinate $x^j=\mathtt m$, let
$x^{j\leftarrow a}$ denote the state obtained by replacing that coordinate
with $a\in\mathcal V$.
Under the masking kernel,
$
   \frac{p_t(x^{j\leftarrow a})}{p_t(x)}
    =
    \frac{\alpha_t}{1-\alpha_t}
    \mathbb P(Y_0^j=a\mid Y_t=x).
$
Thus the reverse dynamics are parameterized by the clean-token posterior.
We approximate this posterior with
$\mu_{\theta,t}^j(\cdot\mid x)\in\Delta(\mathcal V)$, where $\Delta(\mathcal V)$ is the simplex over $\mathcal V$, yielding
$
    \widehat Q_{\theta,T-t}
    \bigl(x^{j\leftarrow a},x\bigr)
    =
    \lambda_t
    \frac{\alpha_t}{1-\alpha_t}
    \mu_{\theta,t}^j(a\mid x).
$

At inference time, we apply a token temperature $\eta>0$ to the denoiser
distribution and define
\begin{equation}
    \mu_{\theta,t}^{j,\eta}(a\mid x)
    =
    \frac{
        \mu_{\theta,t}^j(a\mid x)^{1/\eta}
    }{
        \sum_{a'\in \mathcal V}
        \mu_{\theta,t}^j(a'\mid x)^{1/\eta}
    }.
    \label{eq:temperature_scaled_denoiser}
\end{equation}
Thus, $\eta=1$ recovers the denoiser distribution,
while $\eta<1$ sharpens it and $\eta>1$ flattens it.

The reverse process is simulated on a finite denoising grid
$0=t_0<t_1<\cdots<t_L=T$.
Discretizing the reverse dynamics over each interval
$[t_{\ell-1},t_\ell]$ gives a transition kernel
$
P_{\theta,t_\ell}(y\mid x)
:=
\mathbb{P}_\theta
\bigl(
X_{t_{\ell-1}}=y
\mid
X_{t_\ell}=x
\bigr).$
For masked diffusion, this transition progressively reveals masked
coordinates according to the prescribed noise schedule. Newly revealed
tokens are sampled from
$\mu_{\theta,t_\ell}^{j,\eta}(\cdot\mid x)$, while previously revealed
tokens remain fixed. We treat the fixed sampling temperature $\eta$ as
part of the finite-step sampling procedure and suppress it in the notation
$P_{\theta,t_\ell}$. We write $\mathbb{P}_\theta$ for the resulting path law and
$p_\theta$ for its terminal distribution at $t_0$.

\subsection{Diffusion Guidance via Doob's \texorpdfstring{$h$}{h}-Transform}
\label{sec:doob_transform}
Reward-guided diffusion reweights the terminal distribution of a pretrained
model \citep{wu2024practical,li2024derivative}.
Following the reward-tilting perspective in KL-regularized reinforcement
learning and stochastic control
\citep{kappen2005linear,levine2018reinforcement}, we parameterize the terminal
weight by the reward function $R:\mathcal X\to\mathbb R$ of
Sec.~\ref{sec:introduction}.

\begin{definition}[Guided Terminal Distribution]
\label{def:guided_target}
Given a reward function $R:\mathcal X\to\mathbb R$ and a reward temperature
$\beta>0$, the corresponding guided terminal distribution is
\begin{equation}
\pi_R(x_0)
=
\frac{
p_\theta(x_0)\exp\left(R(x_0)/\beta\right)
}{
Z_R
},
\qquad
Z_R
=
\mathbb E_{X_0\sim p_\theta}
\left[
\exp\left(R(X_0)/\beta\right)
\right].
\label{eq:reward_tilt}
\end{equation}
We refer to the problem of modifying the reverse dynamics of the pretrained diffusion model so that their
terminal marginal is $\pi_R$ as terminal-reweighted diffusion guidance.
\end{definition}

Both soft objectives and hard constraints can be represented through the
choice of $R$. A soft objective assigns a scalar reward $R(x)\in\mathbb R$ to each design,
with larger values indicating greater preference. For a hard constraint specified by a feasible set
$\mathcal A\subseteq\mathcal X$, we use the indicator-based reward function
\begin{equation}
R(x)
=
-P\mathbf{1}\{x\notin\mathcal A\},
\qquad P>0.
\label{eq:hard_reward}
\end{equation}
In the limit $P\to\infty$,
$\exp\left(R(x)/\beta\right)$ reduces to
$\mathbf{1}\{x\in\mathcal A\}$, recovering exact conditioning on the feasible set.
Thus, soft objectives and hard constraints share the same reward-tilting
formulation.

Doob's $h$-transform provides an exact solution to the terminal-reweighting
problem for Markov processes
\citep{rogers2000diffusions,sarkka2019applied}.
More generally, for a positive terminal weight
$\phi:\mathcal X\to\mathbb R_{>0}$, define the corresponding
terminal-reweighted distribution by
\begin{equation}
\pi_\phi(x_0)
=
\frac{p_\theta(x_0)\phi(x_0)}
{\mathbb E_{X_0\sim p_\theta}[\phi(X_0)]}.
\label{eq:general_terminal_reweighting}
\end{equation}
Definition~\ref{def:guided_target} corresponds to the choice
$\phi(x_0)=\exp(R(x_0)/\beta)$. To construct a reverse process with terminal distribution $\pi_\phi$, we propagate the terminal weight $\phi$ backward through the original reverse dynamics using Doob's $h$-function.

\begin{definition}[Doob's $h$-Function]
\label{def:doob_h}
For the terminal weight $\phi$, define the backward value function
\begin{equation}
h_\ell(x)
=
\mathbb E_{\mathbb P_\theta}
\left[
\phi(X_0)
\mid
X_{t_\ell}=x
\right],
\qquad
h_0(x_0)=\phi(x_0).
\label{eq:doob_value}
\end{equation}
\end{definition}

The $h$-function propagates the terminal preference backward through the
original reverse process; it is the soft value function of the corresponding
control problem~\citep{uehara2025inference}. Using this function, Doob's $h$-transform
reweights the reverse transitions so as to realize the desired terminal
distribution.

\begin{proposition}[Doob's $h$-Transform]
\label{prop:doob_transform}
Let $p_\theta$ denote the terminal distribution at $t_0$ of the original
reverse sampler initialized from $X_{t_L}=x_L$. Define the transformed reverse transition kernel by
\begin{equation}
P^h_{\theta,t_\ell}(y\mid x)
=
P_{\theta,t_\ell}(y\mid x)
\frac{h_{\ell-1}(y)}{h_\ell(x)}.
\label{eq:doob_kernel}
\end{equation}
Suppose the reverse sampler is initialized from the same fixed initial state
$X_{t_L}=x_L$ as the original diffusion sampler.
Then $P^h_{\theta,t_\ell}$ is a valid Markov transition kernel, and the
terminal marginal of the transformed chain is exactly $\pi_\phi$.
\end{proposition}
The proof of Proposition~\ref{prop:doob_transform} is deferred to
Appendix~\ref{app:doob_proof}. In masked discrete diffusion, the
reverse sampler always starts from the same fully masked
state $X_{t_L}=\mathtt{m}^N$. Hence, the fixed-initial-state condition of
Proposition~\ref{prop:doob_transform} is naturally satisfied, and
Doob's $h$-transform changes the terminal marginal from $p_\theta$ to
$\pi_\phi$ by reweighting the pretrained reverse transitions.

\begin{figure*}[t]
  \centering
  \captionsetup{skip=3pt}
  \includegraphics[width=\textwidth]{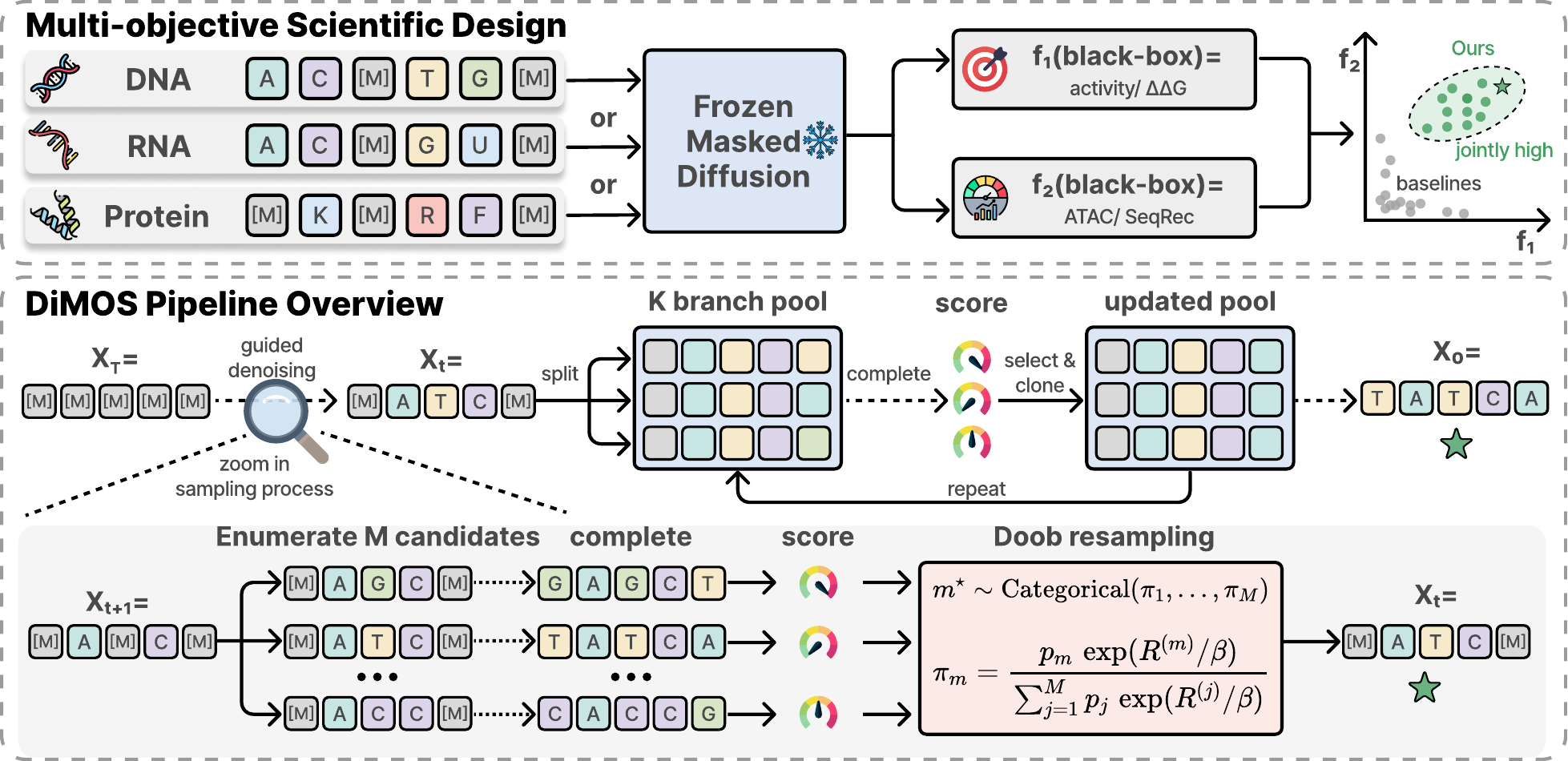}
  \caption{\textbf{Overview of \methodname.}
  \textbf{Top:} multi-objective design with a frozen prior.
  \textbf{Middle:} $K$ branches from step $s$, with reward-guided refinement.
  \textbf{Bottom:} $M$ token candidates are enumerated, completed, and
  stochastically resampled using prior and reward weights.}
  \label{fig:dimos-overview}
  \vspace{-12pt}
\end{figure*}

\section{Method}
\label{sec:method}

In this section, we introduce \methodname for multi-objective scientific
design using a frozen masked diffusion model. We first express the property
requirements as a joint reward (Sec.~\ref{sec:method-reward}). Using this
feedback, approximate Doob-guided local resampling compares candidate updates
and samples the next state without reward gradients
(Sec.~\ref{sec:method-doob}). We then introduce budget-efficient
trajectory search to concentrate computation on promising continuations
(Sec.~\ref{sec:method-search}).

\subsection{Joint Reward for Multiple Requirements}
\label{sec:method-reward}

In scientific design, candidates often need to meet multiple property
thresholds. We therefore instantiate the feasible-set formulation
(hard constraints) of Sec.~\ref{sec:doob_transform} with per-property
thresholds, rather than directly maximizing a weighted sum of per-property
rewards (soft objectives), which does not in general guarantee joint feasibility.
This setting also challenges gradient-based guidance: pass-or-fail
constraints provide no useful gradient signal for correcting violations,
and reward models may themselves be non-differentiable.

% \zqj{is it a reward?}
We consider $n$ reward functions $R_k:\mathcal X\to\mathbb R$, each
with an acceptance threshold $c_k$ and a direction $\epsilon_k\in\{1,-1\}$,
where $1$ denotes a lower-bound requirement and $-1$ an upper-bound
requirement. Define the signed constraint margin
$d_k(x)=\epsilon_k(R_k(x)-c_k)$.
Requirement $k$ is satisfied if and only if $d_k(x)\geq0$.
% \zqj{why not say when $d_k(x)$ is 1}
% $\geq 0 or =1$
The joint feasible set is therefore
$\mathcal A=\{x:\ d_k(x)\geq0\ \text{for all}\ k\}$.
A single indicator of $\mathcal A$ assigns the same penalty to every
infeasible design, regardless of which requirements are satisfied.
To retain feedback about individual requirements, we encode each using
a hard or soft gate and aggregate in log space,
\begin{equation}
    R(x)
    =
    \sum_{k=1}^{n} w_k\,g_k\bigl(d_k(x)\bigr),
    \qquad
    g_k(d)
    =
    \begin{cases}
        -P\,\mathbf 1\{d<0\} & \text{(hard gate)},\\[2pt]
        \log\sigma(d/\tau_k) & \text{(soft gate)},
    \end{cases}
    \label{eq:reward}
\end{equation}
where $\sigma$ is the logistic sigmoid, $w_k>0$ weights each requirement,
and $\tau_k>0$ controls the soft gate's transition width. A hard gate
returns $0$ when the requirement is satisfied and $-P$ otherwise.
A soft gate remains negative and increases continuously toward $0$ as
the margin grows, so it still distinguishes candidates above the threshold.
% \zqj{I can't understand this sentence}
Both gate types use the same feasibility criteria while providing different
reward feedback.
We retain the finite penalty $P>0$ of Eq.~(\ref{eq:hard_reward}), so
violating a hard requirement does not force its weight to zero.
Under the exponential tilt, log-space aggregation yields multiplicative
per-requirement factors (Appendix~\ref{app:hard-gate}).

\subsection{Approximate Doob-Guided Local Resampling}
\label{sec:method-doob}

Doob's $h$-transform connects terminal reward to decisions during denoising.
Using $\phi(x_0)=\exp(R(x_0)/\beta)$ with the joint reward in
Eq.~(\ref{eq:reward}), Proposition~\ref{prop:doob_transform} gives the
relative weight of a candidate next state $y$ from the current state $x$ at
time $t_\ell$,
\begin{equation}
    P^h_{\theta,t_\ell}(y\mid x)
    \propto
    P_{\theta,t_\ell}(y\mid x)\,
    h_{\ell-1}(y),
    \label{eq:doob_relative_weight}
\end{equation}
since the normalizer $h_\ell(x)$ is shared by all candidates. Guidance thus
requires relative future values, but evaluating $h_{\ell-1}(y)$ entails
averaging over all remaining continuations.

\paragraph{Completion-Based Value Estimation.}
To approximate $h_{\ell-1}(y)$, we complete the masked state $y$
to a clean sequence $\hat{x}_0(y)$ using the frozen denoiser,
and use its terminal weight as a proxy~\citep{li2024derivative,uehara2025reward}:
\begin{equation}
    \widehat h_{\ell-1}(y)
    =
    \phi\bigl(\hat{x}_0(y)\bigr)
    =
    \exp\bigl(R(\hat{x}_0(y))/\beta\bigr).
    \label{eq:completion_h}
\end{equation}

\paragraph{Enumeration and Stochastic Resampling.}
To avoid repeated token proposals, we enumerate distinct alternatives at
one newly revealed position.
At a guided step $t_\ell$, fix $x=x_{t_\ell}$ and draw
$Z\sim P_{\theta,t_\ell}(\cdot\mid x)$.
Let $E=\{u:x^u=\mathtt{m},\,Z^u\in\mathcal V\}$ be the reveal set.
When $E\neq\varnothing$, select the position $J\in E$ with the largest
denoiser entropy, breaking ties by index, and enumerate the
$M\leq|\mathcal V|$ most probable tokens $a_1,\ldots,a_M$ under
$\mu_{\theta,t_\ell}^{J,\eta}(\cdot\mid x)$.
For each candidate $y^{(m)}=Z^{J\leftarrow a_m}$, assign
\begin{equation}
    \log\widehat w_m
    =
    \log\mu_{\theta,t_\ell}^{J,\eta}(a_m\mid x)
    +
    \log\widehat h_{\ell-1}(y^{(m)}).
    \label{eq:enumeration_weights}
\end{equation}
We sample
$m^\star\sim\operatorname{Categorical}
(\operatorname{softmax}(\log\widehat w_{1:M}))$
and retain $y^{(m^\star)}$.
Under the factorized token update, all other transition factors are
shared; the token prior appears explicitly because alternatives are
enumerated rather than sampled from it.
Each update uses at most $M$ reward evaluations without reward gradients.

We next bound the discrepancy from the full Doob transition, separating
context-reweighting, top-$M$ truncation, and value-estimation errors.
For this analysis, assume $E\neq\varnothing$ almost surely.

The shared context
$C=(E,(Z^u)_{u\in E\setminus\{J\}})$ records the reveal set and
all newly revealed tokens except the one at $J$.
Its distribution under the base transition is
$\nu_\ell(c\mid x)
=\mathbb{P}_{Z\sim P_{\theta,t_\ell}(\cdot\mid x)}(C=c)$.
The context value
$\Gamma_\ell(c;x)=\mathbb E[h_{\ell-1}(Z)\mid C=c]$
averages over the token at $J$ under the same base transition.
Below, expectations and variances indexed by $\nu_\ell$ are over
$C\sim\nu_\ell(\cdot\mid x)$.

Given $C=c$, the token at $J$ has conditional law
$p_\ell^h(\cdot\mid c,x)$ under the exact Doob transition
$P_{\theta,t_\ell}^h(\cdot\mid x)$.
The enumerated tokens form
$S_M(c)=\{a_1,\ldots,a_M\}\subseteq\mathcal V$.
Restricting the exact conditional to this set and renormalizing gives
\[
    p_{\ell,M}^h(a\mid c,x)
    =
    \frac{
        p_\ell^h(a\mid c,x)\mathbf 1\{a\in S_M(c)\}
    }{
        p_\ell^h(S_M(c)\mid c,x)
    }.
\]
Using $\widehat h_{\ell-1}$ instead induces the practical conditional
$\widehat p_{\ell,M}(\cdot\mid c,x)$, averaged over any completion
randomness. Combining this conditional with the unchanged context law
$\nu_\ell(\cdot\mid x)$ defines the full-state transition
$\widehat P_{\theta,t_\ell}^{(M)}(\cdot\mid x)$.

% Queue the DNA table for the following page without resizing it.
\begin{table*}[t]
  \centering
  \caption{\textbf{DNA design with hard and soft gates.}
  Activity + ATAC results with held-out activity evaluation. Both modes use
  the same success criteria, differing only in whether activity receives a
  hard gate (binary penalty) or a soft gate (smooth reward).}
  \label{tab:dna-main}
  \setlength{\tabcolsep}{3.6pt}
  \renewcommand{\arraystretch}{1.13}
  \resizebox{\textwidth}{!}{%
  \begin{tabular}{lIcIccIccJcIcIcIc}
    \dimosOuterRule
    \multirow{2}{*}{\textbf{Method}}
      & \multirow{2}{*}{\textbf{Joint SR (\%) $\uparrow$}}
      & \multicolumn{2}{cI}{\textbf{Activity $\uparrow$}}
      & \multicolumn{2}{cJ}{\textbf{ATAC $\uparrow$}}
      & \multirow{2}{*}{\textbf{Uniq. (\%)}}
      & \multirow{2}{*}{\textbf{3-mer $r$}}
      & \multirow{2}{*}{\textbf{JASPAR $r$}}
      & \multirow{2}{*}{\textbf{Time (ms)}} \\
      & & \textbf{Median} & \textbf{Top-10}
      & \textbf{Median} & \textbf{Top-10}
      & & & & \\
    \dimosMajorRule
    \rowcolor{dimosTask!70}
    \multicolumn{10}{c}{\color{dimosNavy}\textit{Hard-Gate Mode}} \\
    \baseline{Pretrained}
      & \meanstd{0.00}{0.00}
      & \meanstd{0.18}{0.01}
      & \meanstd{2.41}{0.39}
      & \meanstd{0.03}{0.00}
      & \meanstd{0.65}{0.06}
      & \meanstd{100.00}{0.00}
      & \meanstd{-0.04}{0.01}
      & \meanstd{0.29}{0.01}
      & \meanstd{8.46}{0.23} \\
    \baseline{Best-of-N}
      & \meanstd{2.34}{0.34}
      & \meanstd{1.37}{0.12}
      & \meanstd{5.32}{0.19}
      & \meanstd{0.45}{0.01}
      & \meanstd{0.90}{0.04}
      & \meanstd{100.00}{0.00}
      & \meanstd{0.76}{0.00}
      & \meanstd{0.90}{0.00}
      & \meanstd{271.56}{0.68} \\
    \baseline{TDS}
      & \secondresult{\meanstd{78.80}{6.63}}
      & \secondresult{\meanstd{4.80}{0.13}}
      & \meanstd{6.19}{0.21}
      & \meanstd{0.84}{0.02}
      & \meanstd{0.97}{0.01}
      & \meanstd{48.12}{1.29}
      & \meanstd{0.92}{0.03}
      & \meanstd{0.91}{0.01}
      & \meanstd{171.39}{0.18} \\
    \baseline{SMC}
      & \meanstd{76.04}{8.17}
      & \meanstd{4.74}{0.11}
      & \meanstd{6.25}{0.24}
      & \meanstd{0.83}{0.03}
      & \meanstd{0.97}{0.01}
      & \meanstd{48.02}{2.62}
      & \meanstd{0.92}{0.04}
      & \meanstd{0.91}{0.01}
      & \meanstd{89.28}{0.10} \\
    \baseline{FK-Steering}
      & \meanstd{4.90}{0.58}
      & \meanstd{0.84}{0.06}
      & \meanstd{5.51}{0.14}
      & \meanstd{0.30}{0.01}
      & \meanstd{0.92}{0.01}
      & \meanstd{100.00}{0.00}
      & \meanstd{0.62}{0.01}
      & \meanstd{0.86}{0.00}
      & \meanstd{111.00}{0.83} \\
    \baseline{SVDD}
      & \meanstd{45.05}{0.97}
      & \meanstd{4.14}{0.03}
      & \secondresult{\meanstd{6.62}{0.04}}
      & \secondresult{\meanstd{0.89}{0.01}}
      & \secondresult{\meanstd{0.99}{0.00}}
      & \meanstd{100.00}{0.00}
      & \meanstd{0.69}{0.00}
      & \meanstd{0.91}{0.00}
      & \meanstd{244.61}{0.53} \\
    \dimosMinorRule
    \rowcolor{dimosOurs!70}
    \textbf{\methodname\ (Ours)}
      & \firstresult{\meanstd{87.55}{0.65}}
      & \firstresult{\meanstd{5.34}{0.05}}
      & \firstresult{\meanstd{6.98}{0.07}}
      & \firstresult{\meanstd{0.99}{0.00}}
      & \firstresult{\meanstd{1.00}{0.00}}
      & \meanstd{100.00}{0.00}
      & \meanstd{0.86}{0.00}
      & \meanstd{0.93}{0.00}
      & \meanstd{234.63}{0.33} \\
    \dimosMajorRule
    \rowcolor{dimosTask!70}
    \multicolumn{10}{c}{\color{dimosNavy}\textit{Soft-Gate Mode}} \\
    \baseline{Pretrained}
      & \meanstd{0.00}{0.00}
      & \meanstd{0.18}{0.01}
      & \meanstd{2.41}{0.39}
      & \meanstd{0.03}{0.00}
      & \meanstd{0.65}{0.06}
      & \meanstd{100.00}{0.00}
      & \meanstd{-0.04}{0.01}
      & \meanstd{0.29}{0.01}
      & \meanstd{8.46}{0.23} \\
    \baseline{Best-of-N}
      & \meanstd{2.08}{0.39}
      & \meanstd{1.88}{0.03}
      & \meanstd{5.29}{0.22}
      & \meanstd{0.42}{0.01}
      & \meanstd{0.90}{0.03}
      & \meanstd{100.00}{0.00}
      & \meanstd{0.79}{0.01}
      & \meanstd{0.91}{0.01}
      & \meanstd{240.96}{0.65} \\
    \baseline{TDS}
      & \meanstd{47.40}{7.16}
      & \meanstd{4.14}{0.15}
      & \meanstd{6.15}{0.33}
      & \meanstd{0.79}{0.02}
      & \secondresult{\meanstd{0.95}{0.01}}
      & \meanstd{47.97}{1.55}
      & \meanstd{0.79}{0.02}
      & \meanstd{0.88}{0.02}
      & \meanstd{200.87}{0.93} \\
    \baseline{SMC}
      & \meanstd{41.20}{5.15}
      & \meanstd{4.02}{0.18}
      & \meanstd{5.71}{0.05}
      & \meanstd{0.81}{0.02}
      & \secondresult{\meanstd{0.95}{0.02}}
      & \meanstd{48.39}{1.48}
      & \meanstd{0.74}{0.05}
      & \meanstd{0.86}{0.01}
      & \meanstd{89.73}{0.15} \\
    \baseline{FK-Steering}
      & \meanstd{4.84}{0.34}
      & \meanstd{1.45}{0.07}
      & \meanstd{5.76}{0.12}
      & \meanstd{0.25}{0.00}
      & \meanstd{0.92}{0.01}
      & \meanstd{100.00}{0.00}
      & \meanstd{0.61}{0.02}
      & \meanstd{0.85}{0.01}
      & \meanstd{110.62}{0.32} \\
    \baseline{SVDD}
      & \secondresult{\meanstd{55.16}{1.55}}
      & \secondresult{\meanstd{4.66}{0.01}}
      & \secondresult{\meanstd{7.24}{0.09}}
      & \secondresult{\meanstd{0.83}{0.01}}
      & \firstresult{\meanstd{0.99}{0.00}}
      & \meanstd{100.00}{0.00}
      & \meanstd{0.74}{0.01}
      & \meanstd{0.93}{0.00}
      & \meanstd{244.70}{0.56} \\
    \dimosMinorRule
    \rowcolor{dimosOurs!70}
    \textbf{\methodname\ (Ours)}
      & \firstresult{\meanstd{98.39}{0.29}}
      & \firstresult{\meanstd{5.95}{0.02}}
      & \firstresult{\meanstd{7.32}{0.06}}
      & \firstresult{\meanstd{0.99}{0.00}}
      & \firstresult{\meanstd{0.99}{0.00}}
      & \meanstd{100.00}{0.00}
      & \meanstd{0.88}{0.00}
      & \meanstd{0.93}{0.00}
      & \meanstd{234.60}{0.57} \\
    \dimosOuterRule
  \end{tabular}%
  }
\end{table*}

\begin{theorem}[Full-state error of coordinate-wise Doob resampling]
\label{thm:coordinate_doob_tv}
For any $1\leq M\leq|\mathcal V|$, suppose
$\Gamma_\ell(C;x)>0$ and
$p_\ell^h(S_M(C)\mid C,x)>0$ almost surely under
$\nu_\ell(\cdot\mid x)$.
Then $h_\ell(x)=\mathbb E_{\nu_\ell}[\Gamma_\ell(C;x)]$ and
\begin{equation}
\begin{aligned}
\operatorname{TV}\!\left(
    \widehat P_{\theta,t_\ell}^{(M)}(\cdot\mid x),
    P_{\theta,t_\ell}^h(\cdot\mid x)
\right)&\;\leq
\underbrace{
    \frac12\mathbb E_{\nu_\ell}\!\left[
        \left|
        \frac{\Gamma_\ell(C;x)}{h_\ell(x)}-1
        \right|
    \right]
}_{\text{context-reweighting error}}
+
\underbrace{
    \mathbb E_{\nu_\ell}\!\left[
        p_\ell^h(S_M(C)^c\mid C,x)
    \right]
}_{\text{top-$M$ truncation error}}
\\
&\;+
\underbrace{
    \mathbb E_{\nu_\ell}\!\left[
        \operatorname{TV}\!\left(
            \widehat p_{\ell,M}(\cdot\mid C,x),
            p_{\ell,M}^h(\cdot\mid C,x)
        \right)
    \right]
}_{\text{value-estimation error}}.
\end{aligned}
\label{eq:coordinate_doob_tv}
\end{equation}
The first term is at most
$\sqrt{\operatorname{Var}_{\nu_\ell}[\Gamma_\ell(C;x)]}/(2h_\ell(x))$.
\end{theorem}

The proof and empirical diagnostics are in
Appendices~\ref{app:coordinate_doob_tv} and~\ref{app:local-error}.
When $M=|\mathcal V|$, the truncation term vanishes.
If the completion-based values approximate the exact Doob values well,
the value-estimation term is also small, leaving the discrepancy primarily
controlled by the variation of the shared-context value.

\subsection{Budget-Efficient Trajectory Search}
\label{sec:method-search}

Jointly satisfying multiple requirements can leave a small feasible
region~\citep{xia2024balancing}, making repeated independent sampling costly. We therefore share trajectory prefixes and allocate more search budget to branches with higher estimated joint rewards.

\paragraph{Stage-Wise Branching.}
To control the cost of multi-trajectory exploration, we share a denoising
prefix within each of $B$ independent generation groups, each returning one
design. Each group follows a single trajectory for $s$ steps, then copies
its partially denoised state into $K$ branches. These branches explore
alternative continuations through the stochastic local updates of
Sec.~\ref{sec:method-doob}. The branch width $K$ determines how many
continuations are explored concurrently, while the split step $s$ trades a
longer shared prefix for a shorter branching phase.

\paragraph{Reward-Guided Refinement.}
At post-split checkpoints, we score branches using the completion-based
estimate of Sec.~\ref{sec:method-doob}. Within each group, we copy the state
with the highest estimated joint reward into all $K$ slots, then use
stochastic local updates to explore new continuations. Each group returns
its highest-reward final design. Group-local selection prevents one
trajectory from replacing other groups' outputs, helping preserve sequence
uniqueness. Prefix sharing reduces state updates over $L$ steps from $KL$
to $s+K(L-s)$ per group. Each guided update compares up to $M$ candidates
per branch, with additional reward evaluations for branch selection
(Appendix~\ref{app:algorithm}).

\section{Experiments}
\label{sec:experiments}

In this section, we evaluate \methodname on DNA, protein, and RNA design
using frozen generative priors. We assess joint
success rate (SR) and design quality under hard-gate and soft-gate modes, then examine
component contributions, hyperparameter sensitivity, and time-budget trade-offs.

\subsection{Experimental Setup}
\label{sec:experiments-setup}

\paragraph{Tasks and Gate Modes.}
We study \textbf{\ding{182} Activity + ATAC} for DNA
regulatory activity and chromatin accessibility~\citep{gosai2023machine,wang2024finetuning},
and \textbf{\ding{183} $\Delta\Delta G$ + SeqRec} for protein stability and
sequence recovery. Both tasks require joint threshold satisfaction.
In hard-gate mode,
activity or $\Delta\Delta G$ uses the binary penalty of
Eq.~(\ref{eq:reward}), while the other objective retains its soft gate.
In soft-gate mode, both objectives use soft gates. Success criteria remain
unchanged. This tests guidance when one requirement provides no useful
gradient. Additional DNA, protein, and RNA results appear in
Appendix~\ref{app:additional-tasks}.

\paragraph{Baselines.}
We compare sampling and selection (\textsc{Pretrained}, \textsc{Best-of-N}),
gradient-based steering (\textsc{TDS})~\citep{wu2024practical,nisonoff2024unlocking},
and gradient-free steering (\textsc{SMC}, \textsc{FK-Steering},
\textsc{SVDD})~\citep{singhal2025general,li2024derivative}.
All methods share the frozen prior within each task, and most baselines
have generation times of the same order of magnitude.
Implementation and budget details are in Appendix~\ref{app:baseline-protocol}.

\paragraph{Evaluation Metrics.}
We report joint SR, the fraction of outputs satisfying both thresholds,
and each objective's median over all outputs and over the ten designs with
the highest joint reward. We also report uniqueness (the fraction of distinct
sequences), domain-specific naturalness and validity checks, and generation
time per design. Tabs.~\ref{tab:dna-main} and~\ref{tab:protein-main} report
mean $\pm$ s.d.\ over three seeds. Full metric definitions and thresholds are in
Appendix~\ref{app:evaluation-metrics}.

\subsection{Main Results}
\label{sec:experiments-results}

\begin{table*}[t]
  \centering
  \caption{\textbf{Protein design with hard and soft gates.} $\Delta\Delta G$ + SeqRec results with identical success criteria. The $\Delta\Delta G$ term uses a hard gate (binary penalty) or a soft gate (smooth reward).}
  \label{tab:protein-main}
  \setlength{\tabcolsep}{3.6pt}
  \renewcommand{\arraystretch}{1.13}
  \resizebox{\textwidth}{!}{%
  \begin{tabular}{lIcIccIccJcIcIcIc}
    \dimosOuterRule
    \multirow{2}{*}{\textbf{Method}}
      & \multirow{2}{*}{\textbf{Joint SR (\%) $\uparrow$}}
      & \multicolumn{2}{cI}{\textbf{$\Delta\Delta G$ $\uparrow$}}
      & \multicolumn{2}{cJ}{\textbf{SeqRec $\uparrow$}}
      & \multirow{2}{*}{\textbf{Uniq. (\%)}}
      & \multirow{2}{*}{\textbf{scRMSD (\%)}}
      & \multirow{2}{*}{\textbf{pLDDT (\%)}}
      & \multirow{2}{*}{\textbf{Time (ms)}} \\
      & & \textbf{Median} & \textbf{Top-10}
      & \textbf{Median} & \textbf{Top-10}
      & & & & \\
    \dimosMajorRule
    \rowcolor{dimosTask!70}
    \multicolumn{10}{c}{\color{dimosNavy}\textit{Hard-Gate Mode}} \\
    \baseline{Pretrained}
      & \meanstd{4.04}{0.30}
      & \meanstd{-0.45}{0.00}
      & \meanstd{0.45}{0.03}
      & \meanstd{0.44}{0.00}
      & \meanstd{0.48}{0.00}
      & \meanstd{99.67}{0.05}
      & \meanstd{93.55}{0.11}
      & \meanstd{87.98}{0.11}
      & \meanstd{17.42}{0.08} \\
    \baseline{Best-of-N}
      & \meanstd{16.06}{0.19}
      & \secondresult{\meanstd{-0.02}{0.01}}
      & \meanstd{0.70}{0.05}
      & \meanstd{0.49}{0.00}
      & \meanstd{0.52}{0.00}
      & \meanstd{99.18}{0.03}
      & \meanstd{92.84}{0.09}
      & \meanstd{86.65}{0.16}
      & \meanstd{159.89}{0.05} \\
    \baseline{TDS}
      & \meanstd{13.59}{0.48}
      & \meanstd{-0.12}{0.01}
      & \meanstd{0.38}{0.06}
      & \meanstd{0.47}{0.00}
      & \meanstd{0.50}{0.00}
      & \meanstd{16.75}{0.24}
      & \meanstd{93.12}{1.00}
      & \meanstd{87.96}{0.21}
      & \meanstd{124.49}{0.38} \\
    \baseline{SMC}
      & \meanstd{4.10}{0.30}
      & \meanstd{-0.10}{0.06}
      & \meanstd{0.69}{0.07}
      & \meanstd{0.44}{0.00}
      & \meanstd{0.48}{0.00}
      & \meanstd{22.09}{0.74}
      & \meanstd{92.90}{0.83}
      & \meanstd{89.37}{0.03}
      & \meanstd{83.11}{0.08} \\
    \baseline{FK-Steering}
      & \meanstd{8.03}{0.08}
      & \meanstd{-0.26}{0.01}
      & \meanstd{0.71}{0.03}
      & \meanstd{0.45}{0.00}
      & \meanstd{0.50}{0.00}
      & \meanstd{99.93}{0.00}
      & \meanstd{94.23}{0.16}
      & \meanstd{87.43}{0.19}
      & \meanstd{192.43}{0.12} \\
    \baseline{SVDD}
      & \secondresult{\meanstd{31.68}{0.03}}
      & \firstresult{\meanstd{0.07}{0.01}}
      & \secondresult{\meanstd{0.75}{0.02}}
      & \secondresult{\meanstd{0.54}{0.00}}
      & \secondresult{\meanstd{0.57}{0.00}}
      & \meanstd{98.61}{0.12}
      & \meanstd{93.86}{0.13}
      & \meanstd{84.64}{0.09}
      & \meanstd{166.79}{0.09} \\
    \dimosMinorRule
    \rowcolor{dimosOurs!70}
    \textbf{\methodname\ (Ours)}
      & \firstresult{\meanstd{58.38}{0.36}}
      & \firstresult{\meanstd{0.07}{0.00}}
      & \firstresult{\meanstd{1.38}{0.00}}
      & \firstresult{\meanstd{0.61}{0.00}}
      & \firstresult{\meanstd{0.63}{0.00}}
      & \meanstd{100.00}{0.00}
      & \meanstd{92.32}{0.05}
      & \meanstd{85.92}{0.03}
      & \meanstd{157.63}{0.05} \\
    \dimosMajorRule
    \rowcolor{dimosTask!70}
    \multicolumn{10}{c}{\color{dimosNavy}\textit{Soft-Gate Mode}} \\
    \baseline{Pretrained}
      & \meanstd{4.04}{0.30}
      & \meanstd{-0.45}{0.00}
      & \meanstd{0.45}{0.03}
      & \meanstd{0.44}{0.00}
      & \meanstd{0.48}{0.00}
      & \meanstd{99.67}{0.05}
      & \meanstd{93.55}{0.11}
      & \meanstd{87.98}{0.11}
      & \meanstd{17.38}{0.10} \\
    \baseline{Best-of-N}
      & \meanstd{17.77}{0.19}
      & \meanstd{0.01}{0.01}
      & \meanstd{0.73}{0.04}
      & \meanstd{0.49}{0.00}
      & \meanstd{0.52}{0.00}
      & \meanstd{99.00}{0.08}
      & \meanstd{94.21}{0.05}
      & \meanstd{85.81}{0.05}
      & \meanstd{159.68}{0.13} \\
    \baseline{TDS}
      & \meanstd{9.14}{0.45}
      & \meanstd{-0.31}{0.02}
      & \meanstd{0.41}{0.04}
      & \meanstd{0.47}{0.00}
      & \meanstd{0.49}{0.00}
      & \meanstd{17.34}{0.19}
      & \meanstd{94.88}{0.59}
      & \meanstd{90.67}{0.17}
      & \meanstd{123.78}{0.58} \\
    \baseline{SMC}
      & \meanstd{6.79}{0.43}
      & \meanstd{-0.48}{0.05}
      & \meanstd{0.53}{0.03}
      & \meanstd{0.45}{0.00}
      & \meanstd{0.48}{0.00}
      & \meanstd{22.14}{0.97}
      & \meanstd{92.25}{0.37}
      & \meanstd{86.24}{0.21}
      & \meanstd{82.51}{0.03} \\
    \baseline{FK-Steering}
      & \meanstd{6.38}{0.19}
      & \meanstd{-0.33}{0.00}
      & \meanstd{0.44}{0.01}
      & \meanstd{0.46}{0.00}
      & \meanstd{0.50}{0.00}
      & \meanstd{99.54}{0.00}
      & \meanstd{93.84}{0.11}
      & \meanstd{87.63}{0.11}
      & \meanstd{192.32}{0.06} \\
    \baseline{SVDD}
      & \secondresult{\meanstd{35.66}{0.41}}
      & \secondresult{\meanstd{0.38}{0.01}}
      & \secondresult{\meanstd{1.00}{0.02}}
      & \secondresult{\meanstd{0.53}{0.00}}
      & \secondresult{\meanstd{0.56}{0.00}}
      & \meanstd{96.83}{0.21}
      & \meanstd{96.57}{0.12}
      & \meanstd{83.90}{0.03}
      & \meanstd{166.70}{0.08} \\
    \dimosMinorRule
    \rowcolor{dimosOurs!70}
    \textbf{\methodname\ (Ours)}
      & \firstresult{\meanstd{70.64}{0.24}}
      & \firstresult{\meanstd{0.51}{0.01}}
      & \firstresult{\meanstd{1.39}{0.03}}
      & \firstresult{\meanstd{0.60}{0.00}}
      & \firstresult{\meanstd{0.63}{0.00}}
      & \meanstd{99.93}{0.00}
      & \meanstd{93.99}{0.21}
      & \meanstd{85.61}{0.09}
      & \meanstd{157.51}{0.11} \\
    \dimosOuterRule
  \end{tabular}%
  }
  \par\vspace{-10pt}
\end{table*}

Tabs.~\ref{tab:dna-main} and~\ref{tab:protein-main} summarize DNA and protein
results under both gate modes. Three findings emerge.
\textbf{\ding{182}\enspace Joint Objective Improvement.}
\methodname leads both domains and gate modes in joint SR, including with
a non-differentiable hard gate. Its evaluation-only guidance uses reward
values without smoothing that gate. In soft-gate mode, it nearly doubles
joint SR relative to \textsc{SVDD}, the strongest baseline, while raising
both objective medians, showing distribution-wide gains beyond top-ranked
designs.
\textbf{\ding{183}\enspace Effective Generation-Time Guidance.}
At comparable generation times, \methodname substantially outperforms
\textsc{Best-of-N}, which ranks only completed samples. This supports
integrating reward feedback into generation rather than relying solely on
final selection to find jointly feasible designs within a given budget.
\textbf{\ding{184}\enspace Diverse High-Quality Outputs.}
\methodname retains $100\%$ uniqueness on DNA and over $99.9\%$ on protein
under both gate modes. In DNA hard-gate mode, distinct successful designs
account for $87.55\%$ of its outputs, versus at most $48.12\%$ for
\textsc{TDS}, bounded by its uniqueness. This is consistent with restricting
branch replication to independent generation groups
(Sec.~\ref{sec:method-search}).
\methodname retains high DNA motif-profile correlations and over $92\%$
protein fold recovery (scRMSD), providing quality checks beyond the
optimized objectives.

% Defer the diagnostic figures to the next page, after the main results.
\suppressfloats[t]
\begin{figure}[t]
  \centering
  \includegraphics[width=\textwidth]{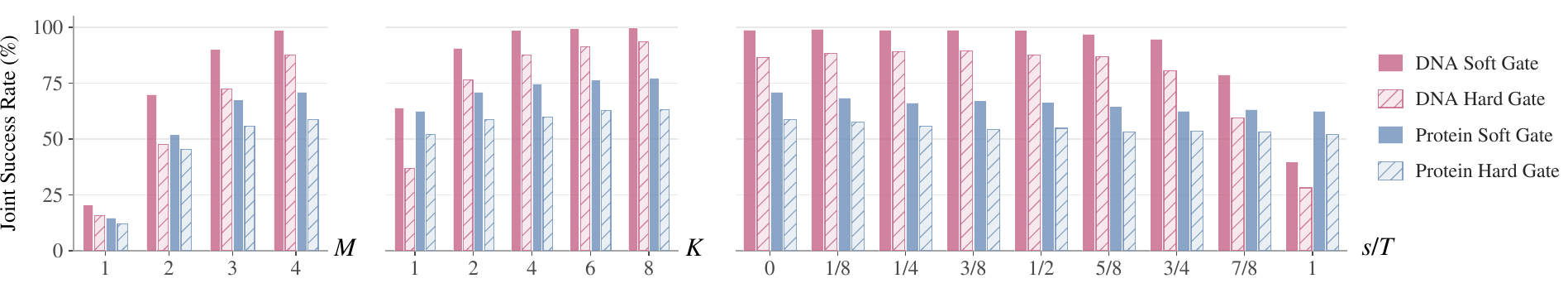}
  \captionsetup{skip=3pt}
  \caption{\textbf{Sensitivity to $M$, $K$, and $s/T$.}
  One parameter varies at a time.}
  \label{fig:dna-parameters}
\end{figure}

% Ablation table + budget envelope, one top-of-page float.
% Generated by _ablation_table_latex.py -- do not edit by hand.
% One float, two panels: the component ablation (left) and the budget envelope (right).
% Top-align the complete panels: table caption above the table, figure caption below
% the plot. The table caption must not reserve a blank row above the right panel.
\begin{figure}[!t]
  \centering
  \begin{minipage}[t]{0.52\textwidth}
    \vspace{0pt}
    \centering
    \captionsetup{position=top,skip=5pt}
    \captionof{table}{\textbf{Component ablation.} Each row removes one component;
    entries are the change in joint SR (\%) from the full configuration, over three seeds.}
    \label{tab:ablation}
    \setlength{\tabcolsep}{1.2pt}%
    \renewcommand{\arraystretch}{1.50}%
    \resizebox{\linewidth}{!}{%
    \begin{tabular}{lIccJcc}
      \dimosOuterRule
      \multirow{2}{*}{\textbf{Configuration}}
        & \multicolumn{2}{cJ}{\textbf{DNA}}
        & \multicolumn{2}{c}{\textbf{Protein}} \\
        & \textbf{Hard Gate} & \textbf{Soft Gate} & \textbf{Hard Gate} & \textbf{Soft Gate} \\
      \dimosMajorRule
      \rowcolor{dimosOurs!60}
      Full \methodname
        & \meanstd{87.55}{0.65}
        & \meanstd{98.39}{0.29}
        & \meanstd{58.38}{0.36}
        & \meanstd{70.64}{0.24} \\
      \dimosMinorRule
      \rowcolor{dimosTask!70}
      \multicolumn{5}{c}{\color{dimosNavy}\textit{Doob-guided local resampling}} \\
      \;\;-- Doob guidance
        & \meanstd{-62.92}{1.84}
        & \meanstd{-64.95}{0.46}
        & \meanstd{-46.96}{0.24}
        & \meanstd{-56.18}{0.35} \\
      \;\;-- enumeration
        & \meanstd{-12.66}{1.72}
        & \meanstd{-5.26}{0.22}
        & \meanstd{-38.28}{0.08}
        & \meanstd{-47.61}{0.08} \\
      \rowcolor{dimosTask!70}
      \multicolumn{5}{c}{\color{dimosNavy}\textit{Budget-efficient trajectory search}} \\
      \;\;-- stage-wise branching
        & \meanstd{-51.61}{2.03}
        & \meanstd{-34.69}{3.29}
        & \meanstd{-6.47}{0.19}
        & \meanstd{-7.75}{0.23} \\
      \;\;-- reward-guided refinement
        & \meanstd{-18.39}{1.63}
        & \meanstd{-7.03}{1.10}
        & \meanstd{-0.11}{0.27}
        & \meanstd{-1.80}{0.08} \\
      \dimosOuterRule
    \end{tabular}}%
  \end{minipage}\hfill
  \begin{minipage}[t]{0.46\textwidth}
    \vspace{0pt}
    \centering
    \includegraphics[width=\linewidth]{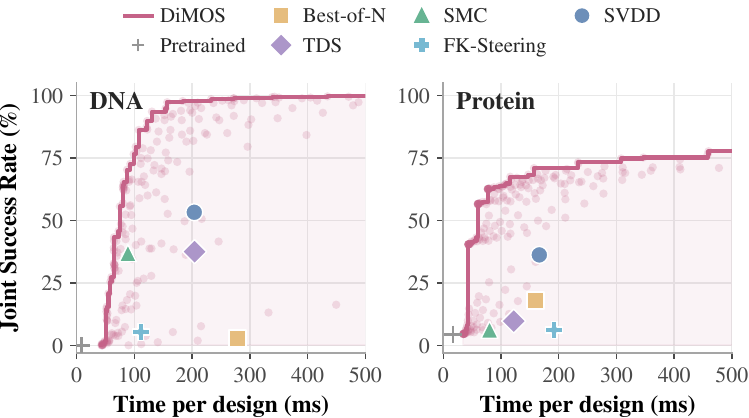}
    \caption{\textbf{Joint success versus generation time.}
    \methodname budget envelopes on Activity + ATAC (DNA) and
    $\Delta\Delta G$ + SeqRec (protein).}
    \label{fig:dna-budget}
  \end{minipage}
  \par\vspace{-18pt}
\end{figure}

\subsection{Hyperparameter Analysis and Ablation}
\label{sec:experiments-allocation}

\paragraph{Hyperparameter Sensitivity.}
Fig.~\ref{fig:dna-parameters} examines candidate width $M$, branch width $K$,
and normalized search onset $s/T$ on DNA and protein. Increasing $M$ and $K$
generally improves joint SR in both gate modes, although gains diminish
as more branches are added. Search onset trades prefix sharing against
multi-trajectory exploration.
DNA maintains high success when branching is delayed through roughly the
first half of denoising. A later split propagates one shared trajectory
instead of $K$ branches for more steps, allowing DNA to save computation on
the prefix with little loss in success. Protein benefits from earlier
branching, favoring more computation on alternative continuations rather
than a long shared prefix. The allocation should therefore reflect the task.

\paragraph{Component Contributions.}
Tab.~\ref{tab:ablation} tests what each component adds beyond the remaining
algorithm.
\textbf{\ding{182}\enspace Local Reward Feedback.}
Removing Doob guidance disables the reward term in local resampling but
retains candidate priors and reward-based branch selection. This causes the
largest drop in joint SR, showing that trajectory-level selection alone
cannot replace reward-informed token updates.
\textbf{\ding{183}\enspace Candidate Construction.}
Replacing enumeration with $M$ i.i.d. token draws at the same selected
position allows repeats while keeping the candidate count and reweighting
rule unchanged. The larger losses on protein highlight the value of
evaluating distinct alternatives, not merely more candidates.
\textbf{\ding{184}\enspace Multiple Continuations.}
Setting $K=1$ retains local guidance but collapses search to a single
trajectory. The larger losses on DNA highlight the value of exploring
multiple continuations rather than committing to one guided path.
\textbf{\ding{185}\enspace Intermediate Selection.}
Disabling refinement retains the $K$ branches and final best-of-$K$
selection but removes intermediate pruning and replication. The losses,
especially on DNA, show that multiple trajectories benefit from redirecting
computation toward promising partial designs during search.
App.~\ref{app:search-transfer} transplants this search layer onto the SVDD
kernel at matched wall clock: it improves that kernel too, but closes under a
fifth of the gap to \methodname, leaving the local transition as the larger
factor.

\subsection{Efficiency under Time Budgets}
\label{sec:experiments-efficiency}

Fig.~\ref{fig:dna-budget} plots an empirical Pareto frontier between
generation time and joint SR from single-seed, soft-gate sweeps. The
staircase gives the highest observed joint SR within each time budget,
revealing two benefits.
\textbf{\ding{182}\enspace Higher Success at a Given Budget.}
At the time budgets of the compared reward-based baselines, \methodname
achieves higher joint SR in both domains. It therefore uses the available
generation budget to produce a higher proportion of jointly feasible designs.
\textbf{\ding{183}\enspace Lower Cost at a Given Success Rate.}
Measured \methodname configurations match or exceed \textsc{SVDD}'s joint
SR with approximately $63\%$ less generation time on DNA and $74\%$ less on
protein. Practitioners can therefore choose higher success within a fixed
budget or lower cost for a target success rate.
\WFclear

\section{Conclusion}
\label{sec:conclusion}

We introduced \methodname, a training-free framework for multi-objective
scientific design with frozen masked diffusion models. It combines
approximate Doob-guided local resampling with budget-efficient trajectory
search, accommodating hard constraints and black-box rewards without
gradients. Experiments across DNA, protein, and RNA show improved joint
success while retaining diverse designs. Allocation studies further highlight
task-dependent trade-offs between candidate width, branching, and search
duration.

% \ifdefined\DiMOSArxiv
%   \section*{Acknowledgements}
%   \input{sec/acknowledgement}
% \else
%   \input{sec/ai_use_statement}
% \fi

\bibliography{iclr2027_conference}
\bibliographystyle{iclr2027_conference}

\clearpage
\appendix
% \section{Limitations}
% \label{sec:limitation}

% \methodname has the following five limitations.

% \begin{itemize}[leftmargin=*,topsep=1pt,itemsep=4pt,parsep=0pt]
%     \item xxx
%     \item xxx
% \end{itemize}

\section{Proof of Proposition~\ref{prop:doob_transform}}
\label{app:doob_proof}

\begin{proof}
By the Markov property of the original reverse process and the definition
of $h_\ell$, for any state $x$ at time $t_\ell$,
\begin{align}
h_\ell(x)
&=
\mathbb E_{\mathbb P_\theta}
\left[
    \phi(X_0)
    \mid
    X_{t_\ell}=x
\right]
\nonumber\\
&=
\mathbb E_{\mathbb P_\theta}
\left[
    h_{\ell-1}(X_{t_{\ell-1}})
    \mid
    X_{t_\ell}=x
\right]
\nonumber\\
&=
\sum_y
P_{\theta,t_\ell}(y\mid x)
h_{\ell-1}(y).
\label{eq:appendix_h_recursion}
\end{align}
Since $\phi$ is strictly positive, so is $h_\ell$ on the states reached by
the reverse sampler. Therefore,
\begin{align}
\sum_y P_{\theta,t_\ell}^h(y\mid x)
&=
\frac{1}{h_\ell(x)}
\sum_y
P_{\theta,t_\ell}(y\mid x)
h_{\ell-1}(y)
=1.
\end{align}
Moreover,
$P_{\theta,t_\ell}^h(y\mid x)\geq0$, so
$P_{\theta,t_\ell}^h(\cdot\mid x)$ is a valid transition kernel.

Let $\mathbb P_\theta^h$ denote the path law induced by the transformed
kernels. For a trajectory
$x_L,x_{L-1},\ldots,x_0$ starting from the fixed state $x_L$,
\begin{align}
&\mathbb P_\theta^h
\left(
    X_{t_{L-1}}=x_{L-1},\ldots,X_{t_0}=x_0
    \mid
    X_{t_L}=x_L
\right)
\nonumber\\
&\qquad=
\prod_{\ell=1}^L
P_{\theta,t_\ell}^h(x_{\ell-1}\mid x_\ell)
\nonumber\\
&\qquad=
\prod_{\ell=1}^L
P_{\theta,t_\ell}(x_{\ell-1}\mid x_\ell)
\frac{h_{\ell-1}(x_{\ell-1})}{h_\ell(x_\ell)}
\nonumber\\
&\qquad=
\left[
    \prod_{\ell=1}^L
    P_{\theta,t_\ell}(x_{\ell-1}\mid x_\ell)
\right]
\frac{\phi(x_0)}{h_L(x_L)},
\label{eq:appendix_path_reweighting}
\end{align}
where the $h$-ratios telescope and $h_0(x_0)=\phi(x_0)$.

Summing over $x_1,\ldots,x_{L-1}$ gives
\begin{align}
\mathbb P_\theta^h
\left(
    X_{t_0}=x_0
    \mid
    X_{t_L}=x_L
\right)
&=
\frac{\phi(x_0)}{h_L(x_L)}
\mathbb P_\theta
\left(
    X_{t_0}=x_0
    \mid
    X_{t_L}=x_L
\right)
\nonumber\\
&=
\frac{
    p_\theta(x_0)\phi(x_0)
}{
    h_L(x_L)
}.
\label{eq:appendix_terminal_intermediate}
\end{align}
By definition of the $h$-function and because the original sampler starts
from the fixed state $x_L$,
\begin{align}
h_L(x_L)
&=
\mathbb E_{\mathbb P_\theta}
\left[
    \phi(X_0)
    \mid
    X_{t_L}=x_L
\right]
\nonumber\\
&=
\sum_{x_0}
p_\theta(x_0)\phi(x_0)
=
\mathbb E_{X_0\sim p_\theta}
[\phi(X_0)].
\end{align}
Substituting this into
Eq.~\eqref{eq:appendix_terminal_intermediate} yields
\[
\mathbb P_\theta^h
\left(
    X_{t_0}=x_0
    \mid
    X_{t_L}=x_L
\right)
=
\frac{
    p_\theta(x_0)\phi(x_0)
}{
    \mathbb E_{X_0\sim p_\theta}[\phi(X_0)]
}
=
\pi_\phi(x_0).
\]
Hence the transformed chain has terminal marginal $\pi_\phi$.
\end{proof}

\clearpage
\section{Proof of Theorem~\ref{thm:coordinate_doob_tv}}\label{app:coordinate_doob_tv}

\begin{proof}
Fix $x=x_{t_\ell}$ and suppress the dependence on $x$ and $t_\ell$
when there is no ambiguity.
Let $A=Z^J$ denote the token at the selected coordinate $J$.
For a context $c=(e,b)$ and token $a\in\mathcal V$, define
$y_x(c,a)$ as the unique successor state whose reveal set is $e$,
whose token at $J$ is $a$, and whose other newly revealed tokens are
given by $b$; all remaining coordinates agree with $x$.
Define the corresponding base conditional by
\[
    p_\ell(a\mid c,x)
    =
    \mathbb{P}(A=a\mid C=c,X_{t_\ell}=x).
\]

For fixed $x$, the correspondence between the successor $Z$ and the pair
$(C,A)$ is one-to-one: $Z$ determines its reveal set $E$, the selected
coordinate $J$, and hence $(C,A)$, while $(C,A)$ reconstructs $Z$ as
$y_x(C,A)$. Therefore,
\begin{equation}
    P_{\theta,t_\ell}\!\left(y_x(c,a)\mid x\right)
    =
    \nu_\ell(c\mid x)\,
    p_\ell(a\mid c,x).
    \label{eq:proof_base_factorization}
\end{equation}

By the definition of the context value,
\[
    \Gamma_\ell(c;x)
    =
    \sum_{a\in\mathcal V}
    p_\ell(a\mid c,x)\,
    h_{\ell-1}\!\left(y_x(c,a)\right).
\]
Therefore, using the backward recursion for the Doob value,
\begin{align}
    h_\ell(x)
    &=
    \sum_y
    P_{\theta,t_\ell}(y\mid x)\,
    h_{\ell-1}(y)
    \notag\\
    &=
    \sum_c
    \nu_\ell(c\mid x)
    \sum_{a\in\mathcal V}
    p_\ell(a\mid c,x)\,
    h_{\ell-1}\!\left(y_x(c,a)\right)
    \notag\\
    &=
    \sum_c
    \nu_\ell(c\mid x)\,
    \Gamma_\ell(c;x)
    =
    \mathbb E_{\nu_\ell}
    \bigl[\Gamma_\ell(C;x)\bigr].
    \label{eq:proof_context_mean}
\end{align}

We next characterize the exact Doob transition.
From Eq.~\eqref{eq:proof_base_factorization},
\begin{align}
    P_{\theta,t_\ell}^{h}
    \!\left(y_x(c,a)\mid x\right)
    &=
    P_{\theta,t_\ell}
    \!\left(y_x(c,a)\mid x\right)
    \frac{
        h_{\ell-1}(y_x(c,a))
    }{
        h_\ell(x)
    }
    \notag\\
    &=
    \nu_\ell(c\mid x)
    \frac{\Gamma_\ell(c;x)}{h_\ell(x)}
    p_\ell^h(a\mid c,x),
    \label{eq:proof_full_doob_factorization}
\end{align}
where
\begin{equation}
    p_\ell^h(a\mid c,x)
    =
    \frac{
        p_\ell(a\mid c,x)
        h_{\ell-1}(y_x(c,a))
    }{
        \Gamma_\ell(c;x)
    }.
    \label{eq:proof_exact_conditional}
\end{equation}
Thus the exact Doob transition changes both the conditional law of the
token at $J$ and the marginal law of the shared context. In particular,
the latter becomes
\begin{equation}
    \nu_\ell^h(c\mid x)
    =
    \nu_\ell(c\mid x)
    \frac{\Gamma_\ell(c;x)}{h_\ell(x)}.
    \label{eq:proof_doob_context}
\end{equation}

To separate the three approximation sources, introduce two intermediate
coordinate-wise kernels.
First, let $\widetilde P_{\theta,t_\ell}^{h}$ keep
$C\sim\nu_\ell(\cdot\mid x)$ unchanged and sample the token at $J$
from the exact conditional $p_\ell^h$:
\begin{equation}
    \widetilde P_{\theta,t_\ell}^{h}
    \!\left(y_x(c,a)\mid x\right)
    =
    \nu_\ell(c\mid x)\,
    p_\ell^h(a\mid c,x).
    \label{eq:proof_oracle_coordinate_kernel}
\end{equation}
Second, let $\widetilde P_{\theta,t_\ell}^{h,(M)}$ additionally restrict
the token to the enumerated set $S_M(c)$:
\begin{equation}
    \widetilde P_{\theta,t_\ell}^{h,(M)}
    \!\left(y_x(c,a)\mid x\right)
    =
    \nu_\ell(c\mid x)\,
    p_{\ell,M}^h(a\mid c,x).
    \label{eq:proof_truncated_kernel}
\end{equation}
Finally, by definition of the practical resampling rule,
\begin{equation}
    \widehat P_{\theta,t_\ell}^{(M)}
    \!\left(y_x(c,a)\mid x\right)
    =
    \nu_\ell(c\mid x)\,
    \widehat p_{\ell,M}(a\mid c,x),
    \label{eq:proof_practical_kernel}
\end{equation}
where $\widehat p_{\ell,M}$ includes the completion-based value
approximation, averaged over any completion randomness.

We now bound the discrepancy between consecutive kernels.

\paragraph{Context-Reweighting Error.}
The kernels $\widetilde P_{\theta,t_\ell}^{h}$ and
$P_{\theta,t_\ell}^{h}$ have the same conditional law
$p_\ell^h(\cdot\mid c,x)$ given $C=c$ and differ only in their context
marginals. Using Eqs.~\eqref{eq:proof_full_doob_factorization}
and~\eqref{eq:proof_oracle_coordinate_kernel},
\begin{align}
&\operatorname{TV}\!\left(
    \widetilde P_{\theta,t_\ell}^{h}(\cdot\mid x),
    P_{\theta,t_\ell}^{h}(\cdot\mid x)
\right)
\notag\\
&\quad=
\frac12
\sum_{c,a}
\nu_\ell(c\mid x)
\left|
    1-\frac{\Gamma_\ell(c;x)}{h_\ell(x)}
\right|
p_\ell^h(a\mid c,x)
\notag\\
&\quad=
\frac12
\mathbb E_{\nu_\ell}\!\left[
    \left|
        \frac{\Gamma_\ell(C;x)}{h_\ell(x)}-1
    \right|
\right].
\label{eq:proof_context_error}
\end{align}

\paragraph{Top-$M$ Truncation Error.}
For each context $c$, $p_{\ell,M}^h$ is obtained by restricting
$p_\ell^h$ to $S_M(c)$ and renormalizing. Hence
\begin{equation}
    \operatorname{TV}\!\left(
        p_{\ell,M}^h(\cdot\mid c,x),
        p_\ell^h(\cdot\mid c,x)
    \right)
    =
    p_\ell^h(S_M(c)^c\mid c,x).
    \label{eq:proof_truncation_identity}
\end{equation}
Indeed, writing
$q_c=p_\ell^h(S_M(c)\mid c,x)$,
\begin{align*}
&\frac12
\sum_{a\in\mathcal V}
\left|
    p_{\ell,M}^h(a\mid c,x)
    -
    p_\ell^h(a\mid c,x)
\right|
\\
&\quad=
\frac12
\left[
    \sum_{a\in S_M(c)}
    p_\ell^h(a\mid c,x)
    \left(\frac1{q_c}-1\right)
    +
    \sum_{a\notin S_M(c)}
    p_\ell^h(a\mid c,x)
\right]
\\
&\quad=
\frac12\bigl[(1-q_c)+(1-q_c)\bigr]
=
1-q_c.
\end{align*}
Since the two full-state kernels share the same context law
$\nu_\ell$, Eq.~\eqref{eq:proof_truncation_identity} gives
\begin{align}
&\operatorname{TV}\!\left(
    \widetilde P_{\theta,t_\ell}^{h,(M)}(\cdot\mid x),
    \widetilde P_{\theta,t_\ell}^{h}(\cdot\mid x)
\right)
\notag\\
&\quad=
\mathbb E_{\nu_\ell}\!\left[
    p_\ell^h(S_M(C)^c\mid C,x)
\right].
\label{eq:proof_truncation_error}
\end{align}

\paragraph{Value-Estimation Error.}
The practical and truncated-oracle kernels also share the same context
law $\nu_\ell$. Therefore,
\begin{align}
&\operatorname{TV}\!\left(
    \widehat P_{\theta,t_\ell}^{(M)}(\cdot\mid x),
    \widetilde P_{\theta,t_\ell}^{h,(M)}(\cdot\mid x)
\right)
\notag\\
&\quad=
\frac12
\sum_c\nu_\ell(c\mid x)
\sum_{a\in\mathcal V}
\left|
    \widehat p_{\ell,M}(a\mid c,x)
    -
    p_{\ell,M}^h(a\mid c,x)
\right|
\notag\\
&\quad=
\mathbb E_{\nu_\ell}\!\left[
    \operatorname{TV}\!\left(
        \widehat p_{\ell,M}(\cdot\mid C,x),
        p_{\ell,M}^h(\cdot\mid C,x)
    \right)
\right].
\label{eq:proof_value_error}
\end{align}

Applying the triangle inequality along
\[
    \widehat P_{\theta,t_\ell}^{(M)}
    \;\longrightarrow\;
    \widetilde P_{\theta,t_\ell}^{h,(M)}
    \;\longrightarrow\;
    \widetilde P_{\theta,t_\ell}^{h}
    \;\longrightarrow\;
    P_{\theta,t_\ell}^{h}
\]
and substituting
Eqs.~\eqref{eq:proof_context_error},
\eqref{eq:proof_truncation_error}, and
\eqref{eq:proof_value_error}
yields Eq.~\eqref{eq:coordinate_doob_tv}.

Finally, Eq.~\eqref{eq:proof_context_mean} implies
$\mathbb E_{\nu_\ell}[\Gamma_\ell(C;x)]=h_\ell(x)$.
By Cauchy--Schwarz,
\begin{align}
\frac12
\mathbb E_{\nu_\ell}\!\left[
    \left|
        \frac{\Gamma_\ell(C;x)}{h_\ell(x)}-1
    \right|
\right]
&=
\frac{
    \mathbb E_{\nu_\ell}
    [|\Gamma_\ell(C;x)-h_\ell(x)|]
}{
    2h_\ell(x)
}
\notag\\
&\leq
\frac{
    \sqrt{
        \operatorname{Var}_{\nu_\ell}
        [\Gamma_\ell(C;x)]
    }
}{
    2h_\ell(x)
}.
\end{align}
This proves the claimed bound.
\end{proof}

The branch-level pruning of Sec.~\ref{sec:method-search} sits outside this
bound: it serves search rather than exact sampling from the
terminal-reweighted distribution.

\section{Empirical Evaluation of Local Approximation Error}
\label{app:local-error}

We evaluate the single-step approximation in
Theorem~\ref{thm:coordinate_doob_tv} on DNA Activity + ATAC, protein
$\Delta\Delta G$ + sequence recovery, and RNA MRL + accessibility.
For each domain and gate mode, we collect $24$ partially generated
sequences from three seeds and eight generation groups per seed.
Each group contributes its first branch when three MASK positions remain,
using the main-table configurations. Protein samples are balanced across
the first six held-out backbones reaching this count.

This choice makes the reference calculation tractable. Exact $h$-values
average terminal exponential rewards over all remaining base-process
continuations, whose number grows rapidly with the number of MASK
positions. Earlier in decoding, Monte Carlo estimation is more affordable
but introduces sampling uncertainty, especially when a few continuations
dominate the reward weights. Using three MASK positions permits exact
enumeration across all three domains, isolating the sampler's approximation
error from Monte Carlo reference error at a common near-terminal setting.

We compute the reference Doob transition by enumerating the remaining
base-process continuations and compare it with the practical local
resampling kernel. The reference includes all token classes, while
practical menus and completion rules remain unchanged. Completion
randomness is integrated where applicable, so the reported TV measures
the full single-step discrepancy rather than only the context term.
Production-code replays verify the practical resampling probabilities.

% Source: DOIT_Multi/experiments/THEORY_UNIFORM24_20260925/report/summary.csv
\begin{table}[htbp]
  \centering
  \caption{\textbf{Local approximation error across domains.}
  Each row summarizes $24$ partially generated sequences with three
  remaining MASK positions.}
  \label{tab:local-error}
  \small
  \setlength{\tabcolsep}{8pt}
  \renewcommand{\arraystretch}{1.15}
  \begin{tabular}{llIcrr}
    \dimosOuterRule
    \textbf{Domain} & \textbf{Gate} & \textbf{Samples}
      & \textbf{Median TV} & \textbf{Mean TV} \\
    \dimosMajorRule
    DNA & Soft & 24 & 0.00041 & 0.00286 \\
    DNA & Hard & 24 & 0.00012 & 0.00968 \\
    \dimosMinorRule
    Protein & Soft & 24 & 0.00073 & 0.09231 \\
    Protein & Hard & 24 & 0.00031 & 0.03408 \\
    \dimosMinorRule
    RNA & Soft & 24 & 0.02280 & 0.03571 \\
    RNA & Hard & 24 & 0.00123 & 0.01503 \\
    \dimosOuterRule
  \end{tabular}
\end{table}

Tab.~\ref{tab:local-error} shows median TV below $0.001$ for DNA and
protein under both gate modes and below $0.023$ for RNA, indicating close
agreement for typical sampled updates in this near-terminal regime.

\section{Why the Hard-Gate Penalty Is Finite}
\label{app:hard-gate}

The log-space aggregation in Eq.~(\ref{eq:reward}) gives the terminal weight
\begin{equation}
    \exp\bigl(R(x)/\beta\bigr)
    =
    \prod_{k=1}^{n}
    \Bigl[
        \exp\bigl(g_k(d_k(x))\bigr)
    \Bigr]^{w_k/\beta}.
    \label{eq:reward_product}
\end{equation}
A soft gate contributes $\sigma(d_k(x)/\tau_k)^{w_k/\beta}$, while a hard
gate contributes $1$ when satisfied and $e^{-Pw_k/\beta}$ otherwise. Thus a
finite hard penalty discourages violations without excluding them from the
support. Two properties of this construction shape the sampler.

\paragraph{A Constant Penalty Cancels.}
When all candidate completions violate the same hard-gated requirement,
its common contribution $-Pw_k/\beta$ cancels during local normalization.
Relative weights then depend on the prior and the remaining requirements.
A hard gate distinguishes candidates on opposite sides of its threshold,
but provides no ranking on either side. A soft gate additionally
distinguishes their margins.

\paragraph{The Penalty Is Finite.}
A finite penalty keeps local resampling well-defined when all candidate
completions violate at least one hard-gated requirement, allowing the
remaining reward feedback to guide selection. In the limit $P\to\infty$,
these candidates receive zero weight and the local weights cannot be
normalized. The
strength of violation suppression depends on $Pw_k/\beta$ relative to the
other contributions to the log weight, rather than on $P$ alone.

\section{Algorithm and Implementation Details}
\label{app:algorithm}

In this section, we detail the implementation of \methodname. We first
specify the sampling schedule and transition notation, then present the
full algorithm and explain how completions are used for local guidance
and branch scoring.

\subsection{Sampling Schedule and Transition Notation}
\label{app:sampling_schedule}

Algorithm~\ref{alg:dimos} describes local enumeration and branch search
for the main-text experiments. We index denoising steps by $i=0,\ldots,L-1$,
so branching occurs before the first update when $s=0$. The corresponding noise index is
$\ell=L-i$. Token enumeration is enabled from the first guided update,
independently of the split step. The configured
guidance window is $\mathcal{G}$; branch-selection checkpoints are
$\mathcal{C}=\{i:s<i<L-1,\ (i-s)\bmod f=0\}$, where $f$ is the search
interval. The final update is unguided in the default configurations.

\textsc{Reveal} returns a shared base successor $z$ and its newly revealed
positions $E$: it chooses a count using the configured reveal schedule,
selects that many remaining masked positions uniformly, and samples their
tokens. Previously committed tokens stay fixed. We write $q_i^j$ for the
actual conditional token distribution after proposal tempering, and $Q_i$
for the corresponding IID transition. For enumeration, all candidates share
$z$ except at the selected position, so only the factor $q_i^j(a_m)$ varies
in their prior weights. When $E$ is empty, the code
uses IID candidates from $Q_i$ and reward-only weights, without multiplying
their proposal probability a second time.

\begin{table}[t]
  \centering
  \caption{\textbf{Sampler configuration of the reported \methodname\ runs}, one row per task.
  $T$ is the number of executed denoising steps, $M$ the candidate width, $K$ the branch width,
  $s$ the split step, $\beta$ the tilt temperature of Eq.~(\ref{eq:doob_value}), and $\eta$ the
  proposal temperature.}
  \label{tab:sampler-config}
  \small
  \setlength{\tabcolsep}{8pt}
  \begin{tabular}{llcccccc}
    \toprule
    \textbf{Domain} & \textbf{Task} & $T$ & $M$ & $K$ & $s$ & $\beta$ & $\eta$ \\
    \midrule
    \multirow{2}{*}{DNA}     & ATAC          & $128$ & $4$ & $4$ & $64$ & $0.25$ & $3.0$  \\
                             & MaxGap        & $128$ & $4$ & $8$ & $64$ & $0.25$ & $3.0$  \\
    \cmidrule(lr){1-8}
    \multirow{2}{*}{Protein} & SeqRec        & $49$  & $4$ & $2$ & $0$  & $0.02$ & $1.0$  \\
                             & GRAVY         & $49$  & $2$ & $2$ & $0$  & $0.02$ & $1.0$  \\
    \cmidrule(lr){1-8}
    \multirow{2}{*}{RNA}     & Accessibility & $49$  & $4$ & $2$ & $18$ & $0.04$ & $0.05$ \\
                             & MFE           & $49$  & $4$ & $2$ & $18$ & $0.04$ & $1.0$  \\
    \bottomrule
  \end{tabular}
\end{table}

\subsection{DIMOS Algorithm}
\label{app:dimos_algorithm}

\begin{algorithm}[ht]
\caption{\methodname with local token enumeration and branch search}
\label{alg:dimos}
\small
\begin{algorithmic}[1]
\Require frozen denoiser $\mu_{\theta}$; joint reward $R$;
$B$ lineages, $L$ steps, $M$ tokens, $K$ branches; split $s$;
guidance steps $\mathcal{G}$, search checkpoints $\mathcal{C}$;
temperatures $\eta,\beta$
\Ensure $B$ clean designs
\State initialize one fully masked state in each of $B$ independent lineages
\For{$i=0,\dots,L-1$}
  \State $\ell\gets L-i$
  \If{$i=s$ \textbf{and} $K>1$}
    \State clone each lineage's state into $K$ branches
  \EndIf
  \ForAll{active states $x$ \textbf{in parallel}}
    \If{$i\in\mathcal{G}$ \textbf{and} $x$ contains masks}
      \State compute tempered token conditionals $q_i$ from
             $\mu_{\theta}(x,t_\ell)$ using $\eta$
      \State $(z,E)\gets\textsc{Reveal}(x,q_i,i)$
      \If{$E\ne\varnothing$} \Comment{enumerated menu}
        \State $j\gets\arg\max_{u\in E} H(q_i^u)$
        \State $(a_1,\ldots,a_M)\gets\operatorname{TopM}(q_i^j)$
               \Comment{distinct tokens}
        \State $y^{(m)}\gets z^{j\leftarrow a_m}$,
               $b_m\gets\log q_i^j(a_m)$ for each $m$
      \Else \Comment{empty-reveal fallback}
        \State draw $y^{(1)},\ldots,y^{(M)}$ independently from
               $Q_i(\cdot\mid x)$; $b_m\gets0$
      \EndIf
      \State $\hat{x}^{(m)}_0\gets\textsc{Complete}(y^{(m)},t_{\ell-1})$;
             $R^{(m)}\gets R(\hat{x}^{(m)}_0)$
      \State $\log\widehat w_m\gets b_m+R^{(m)}/\beta$
      \State $m^\star\sim\operatorname{Categorical}
             (\operatorname{softmax}(\log\widehat w_{1:M}))$;
             $x\gets y^{(m^\star)}$
    \Else
      \State advance $x$ with the unguided reverse update
    \EndIf
  \EndFor
  \If{$K>1$ \textbf{and} $i\in\mathcal{C}$}
    \State compute each branch's reward as
           $R(\textsc{Complete}_{\mathrm{argmax}}(x))$
    \State within each lineage, copy the highest-reward branch
           into all $K$ slots
  \EndIf
\EndFor
\State apply the configured final denoising to any residual masks
\State \Return the highest-reward clean branch from each lineage
\end{algorithmic}
\end{algorithm}

\subsection{Completion and Branch Reward Estimation}
\label{app:completion_details}

\textsc{Complete} preserves committed tokens and fills only remaining masks
using the configured completion rule: stochastic completion with shared
random numbers across a parent's candidates for DNA, and argmax completion
for protein/RNA. Denoiser outputs are reused for branch reward estimation where
available; a reward evaluation is reused only when the completion is the
same. Branch selection uses an interval of $f=1$.

\section{Evaluation Criteria}
\label{app:task-thresholds}
\label{app:evaluation-metrics}

In this section, we define the task-specific success thresholds and the
metrics used to assess joint feasibility, design quality, diversity,
validity, and efficiency.

\subsection{Task Success Thresholds}

Tab.~\ref{tab:task-thresholds} lists the final evaluation criteria,
shared across methods and gate modes within each task.

\begin{table}[ht]
  \centering
  \caption{\textbf{Canonical joint-success criteria for the six design tasks.} Both
  objective conditions must be satisfied.}
  \label{tab:task-thresholds}
  \small
  \setlength{\tabcolsep}{5.5pt}
  \renewcommand{\arraystretch}{1.05}
  \begin{tabular}{lIll}
    \dimosOuterRule
    \textbf{Task} & \textbf{Objective 1} & \textbf{Objective 2} \\
    \dimosMajorRule
    Activity + ATAC
      & Activity $\geq 4.0$ & ATAC $\geq 0.7$ \\
    Activity + MaxGap
      & Activity $\geq 4.0$ & MaxGap $\geq 1.0$ \\
    \dimosSectionRule
    ddG + SeqRec
      & ddG $\geq 0$ & SeqRec $\geq 0.55$ \\
    ddG + GRAVY
      & ddG $\geq 0$ & GRAVY $\leq -0.70$ \\
    \dimosSectionRule
    MRL + Opening dG
      & Opening dG $\leq 0.25$ kcal/mol & Optimus-MRL $\geq 0.45$ \\
    MRL + MFE
      & MFE $\leq -16$ kcal/mol & Optimus-MRL $\geq 0.45$ \\
    \dimosOuterRule
  \end{tabular}
\end{table}

\subsection{Metric Definitions}

\textbf{\ding{182} Joint Success Rate (SR)} is the primary metric: the
fraction of returned designs satisfying both canonical objective thresholds,
rather than optimizing one objective at the expense of the other.
For RNA, joint SR uses the held-out Optimus-MRL reward
model~\citep{sample2019human} together with the task-specific structural
criterion, while generation is guided by ConvGRU-MRL.

\textbf{\ding{183} Objective Quality} reports each objective's median over
all outputs and its median among the ten designs with the highest joint reward,
capturing both distribution-wide steering and best-design quality. Top-10
designs are ranked by the joint reward rather than either objective alone.

\textbf{\ding{184} Diversity and Validity} include exact-sequence
uniqueness, DNA naturalness, and protein fold recovery.
Uniqueness is the percentage of distinct exact sequences.
DNA validity uses 3-mer and JASPAR 2022 motif-profile correlations
\citep{castromondragon2022jaspar} to the top-$0.1\%$ HepG2 reference set.
Protein validity uses held-out ESMFold
\citep{lin2023evolutionary}, reporting the fractions with scRMSD $<2$\AA{}
and pLDDT $>80$. These checks
expose particle collapse and reward hacking beyond the guidance reward models.

\textbf{\ding{185} Efficiency} is generation time per returned design,
reported in milliseconds and not ranked. Within any single table all timings
come from one measurement batch
run with a single process per GPU and with kernel compilation disabled, so
that wall clock is not distorted by contention and so that a speed setting
never changes a reported result. For \textsc{Best-of-N}, we select the tested
candidate count whose measured generation time is closest to \methodname's
within each task and gate mode.

All main-table entries are mean $\pm$ population standard deviation over
three independent seeds.

\section{Baseline and Budget Protocol}
\label{app:baseline-protocol}

We use the pretrained DNA and structure-conditioned protein generators
released with DRAKES~\citep{wang2024finetuning}. DNA designs are 200-bp
sequences in the GOSAI setting~\citep{gosai2023machine}. Protein generation
is conditioned on 12 held-out backbones.

\textsc{Pretrained} samples directly from the frozen prior.
\textsc{Best-of-N} draws independent samples and selects by terminal joint
reward; its candidate count is calibrated to the closest \methodname
wall-clock budget. \textsc{TDS} combines a gradient-guided proposal with
sequential importance correction. \textsc{SMC} uses the pretrained reverse
kernel as its proposal and resamples particles; \textsc{FK-Steering} uses the
published discrete Feynman--Kac difference potential; and \textsc{SVDD-PM}
greedily selects local candidates using posterior-mean values. RNA excludes
\textsc{TDS} because neither selected task supplies a differentiable full joint
reward.

For the gated-reward comparisons, \methodname\ and all reward-based
baselines share the same reward parameters within each task and gate mode:
$P=30$, $w_k=1$ for every requirement, and the soft-gate widths in
Tab.~\ref{tab:reward-config}. In soft-gate mode, both requirements use soft
gates. In hard-gate mode, Activity (DNA), $\Delta\Delta G$ (protein), or
MRL (RNA) uses the binary $0/-P$ gate, while the other requirement retains
its soft gate. The width $\tau_k$ is unused for a hard gate.
\textsc{Pretrained} does not use rewards during generation.

\begin{table}[ht]
  \centering
  \caption{\textbf{Shared soft-gate widths.} The widths follow the reward order shown.}
  \label{tab:reward-config}
  \small
  \setlength{\tabcolsep}{8pt}
  \begin{tabular}{llcc}
    \toprule
    \textbf{Domain} & \textbf{Reward Pair} & $\tau_1$ & $\tau_2$ \\
    \midrule
    DNA     & Activity + ATAC            & $0.5$  & $0.05$ \\
            & Activity + MaxGap          & $0.5$  & $0.5$  \\
    \cmidrule(lr){1-4}
    Protein & $\Delta\Delta G$ + SeqRec   & $0.3$  & $0.03$ \\
            & $\Delta\Delta G$ + GRAVY    & $0.3$  & $0.10$ \\
    \cmidrule(lr){1-4}
    RNA     & MRL + Opening $\Delta G$   & $0.05$ & $0.25$ \\
            & MRL + MFE                  & $0.05$ & $1.0$  \\
    \bottomrule
  \end{tabular}
\end{table}

Every baseline runs at the defaults published with its own implementation.
\textsc{Best-of-N} and \textsc{SVDD} expose a compute budget through their
candidate counts $K$ and $M$, which we set per block to the closest match to
the \methodname\ generation budget that is not cheaper than it.

Every method returns the same number of final designs within each domain:
$N=640$ DNA sequences per seed, $N=128$ protein sequences per wild-type
backbone, and $N=128$ RNA sequences per seed. Task definitions, frozen priors,
and reward models used for evaluation are fixed across methods. DNA activity
and protein ddG rewards are computed using model checkpoints separate from
those used for guidance. Recorded generation wall time is normalized by the
number of returned designs.

\section{Separating the Search Layer from the Local Kernel}
\label{app:search-transfer}

Tab.~\ref{tab:ablation} removes components from \methodname. This section runs
the complementary test: it gives the search layer of
Sec.~\ref{sec:method-search} to a different local kernel. We keep SVDD-PM at
its released hyperparameters and change only where computation is spent.
Branches are scored with SVDD-PM's own posterior-mean completion, so no
component of Sec.~\ref{sec:method-doob} enters. Solving the search arm's cost
for the \methodname\ oracle budget gives $M=4$, which is also
\methodname's own candidate width, so the matched row and \methodname\ differ
only in the per-step transition.

\begin{table}[t]
  \centering
  \caption{\textbf{Separating the search layer from the local kernel.} DNA
  activity--ATAC, $N=640$, three seeds. Rows marked Tab.~\ref{tab:dna-main} are
  reproduced from the main results. Oracle counts reward-model forward passes
  per seed; time is generation wall clock per returned design.}
  \label{tab:search-transfer}
  \small
  \setlength{\tabcolsep}{6pt}
  \begin{tabular}{lccccccc}
    \toprule
    \textbf{Method} & $M$ & $K$ & $s$
      & \textbf{Joint SR (\%) $\uparrow$} & \textbf{3-mer $r$}
      & \textbf{Oracle} & \textbf{Time (ms)} \\
    \midrule
    \multicolumn{8}{c}{\textit{Soft-Gate Mode}} \\
    \baseline{SVDD-PM} (Tab.~\ref{tab:dna-main})
      & $12$ & $1$ & --   & \meanstd{55.16}{1.55} & $0.74$ & $1.97$M & $244.7$ \\
    \;\;$+$ search, matched budget
      & $4$  & $4$ & $64$ & \meanstd{63.02}{2.20} & $0.73$ & $1.96$M & $251.0$ \\
    \;\;$+$ search, released $M$
      & $20$ & $4$ & $64$ & \meanstd{84.43}{0.96} & $0.79$ & $8.51$M & $1101.8$ \\
    \methodname\ (Tab.~\ref{tab:dna-main})
      & $4$  & $4$ & $64$ & \meanstd{98.39}{0.29} & $0.88$ & $1.94$M & $234.6$ \\
    \midrule
    \multicolumn{8}{c}{\textit{Hard-Gate Mode}} \\
    \baseline{SVDD-PM} (Tab.~\ref{tab:dna-main})
      & $12$ & $1$ & --   & \meanstd{45.05}{0.97} & $0.69$ & $1.97$M & $244.6$ \\
    \;\;$+$ search, matched budget
      & $4$  & $4$ & $64$ & \meanstd{52.34}{0.80} & $0.65$ & $1.96$M & $250.7$ \\
    \;\;$+$ search, released $M$
      & $20$ & $4$ & $64$ & \meanstd{72.86}{1.41} & $0.71$ & $8.51$M & $1101.8$ \\
    \methodname\ (Tab.~\ref{tab:dna-main})
      & $4$  & $4$ & $64$ & \meanstd{87.55}{0.65} & $0.86$ & $1.94$M & $234.6$ \\
    \bottomrule
  \end{tabular}
\end{table}

The search layer transfers. At $1.03\times$ the SVDD-PM generation time and a
slightly smaller oracle budget, it adds $7.86$ and $7.29$ points on the soft
and hard gates.

The local kernel accounts for the larger share. The matched row and
\methodname\ share $M$, $K$, and $s$ and differ only in how one step selects
its next state, and that difference is $35.37$ and $35.21$ points. Raising the
SVDD candidate width until the search arm spends $4.39\times$ the \methodname\
oracle budget and $4.70\times$ its time still leaves it $13.96$ and $14.69$
points behind. Naturalness moves the same way: the transferred search costs the
SVDD kernel some 3-mer correlation ($0.74$ to $0.73$ soft, $0.69$ to $0.65$
hard), while \methodname\ attains the highest value in both modes.

Tab.~\ref{tab:svdd-width} places these arms against the only budget knob the
SVDD kernel has on its own. Joint success rises monotonically with the
candidate width $M$, but the marginal return falls from $5.0$ points per unit
$M$ between $M=4$ and $M=8$ to $0.9$ between $M=16$ and $M=20$, so buying
joint feasibility through candidates alone becomes expensive well before the
gap to \methodname\ closes. The matched search arm comes within $0.7$ points
of the $M=20$ row on the soft gate while spending $0.60\times$ its oracle
budget and $0.58\times$ its time, and lands near $M=17$ on the hard gate.
Reallocating a fixed budget is therefore worth more to this kernel than
widening its candidate set.

% Generated by _svdd_curve_latex.py -- do not edit by hand.
\begin{table}[t]
  \centering
  \caption{\textbf{SVDD-PM candidate width on DNA activity--ATAC.} $T=128$, $N=640$, three seeds,
  no search layer. Oracle and time are shared by the two gate modes to within $0.2\%$.}
  \label{tab:svdd-width}
  \small
  \setlength{\tabcolsep}{7pt}
  \begin{tabular}{lcccccc}
    \toprule
    \multirow{2}{*}{$M$} & \multirow{2}{*}{\textbf{Oracle}} & \multirow{2}{*}{\textbf{Time (ms)}}
      & \multicolumn{2}{c}{\textbf{Joint SR (\%) $\uparrow$}} & \multicolumn{2}{c}{\textbf{3-mer $r$}} \\
    \cmidrule(lr){4-5}\cmidrule(lr){6-7}
    & & & Soft & Hard & Soft & Hard \\
    \midrule
    $4$ & $0.66$M & $84$ & \meanstd{26.30}{0.83} & \meanstd{18.59}{1.38} & $0.69$ & $0.64$ \\
    $8$ & $1.31$M & $159$ & \meanstd{46.25}{1.00} & \meanstd{37.71}{1.36} & $0.72$ & $0.68$ \\
    $10$ & $1.64$M & $184$ & \meanstd{51.51}{1.67} & \meanstd{43.07}{0.70} & $0.73$ & $0.69$ \\
    $12$\;\;{\scriptsize(Tab.~\ref{tab:dna-main})} & $1.97$M & $245$ & \meanstd{55.16}{1.55} & \meanstd{45.05}{0.97} & $0.74$ & $0.69$ \\
    $16$ & $2.62$M & $336$ & \meanstd{60.21}{2.78} & \meanstd{51.56}{1.67} & $0.75$ & $0.71$ \\
    $20$ & $3.28$M & $434$ & \meanstd{63.70}{1.09} & \meanstd{55.16}{1.85} & $0.76$ & $0.71$ \\
    \bottomrule
  \end{tabular}
\end{table}

\section{Additional Tasks}
\label{app:additional-tasks}

In this section, we extend the evaluation to DNA cell-type specificity,
protein hydrophilicity, and RNA translation paired with accessibility or
folding stability. We examine these additional requirements under both
gate modes.

\subsection{DNA Cell-Type Specificity}
\label{app:dna-maxgap}

Tab.~\ref{tab:dna-maxgap} reports the second DNA task, \emph{Activity +
MaxGap}, in which HepG2 activity is paired with a specificity margin against
K562 and SK-N-SH. Unlike \emph{Activity + ATAC}, the two objectives are in
tension: in the reference enhancer set activity and specificity are negatively
correlated, so the direction that raises expression in HepG2 tends to raise it
in the off-target lines as well. The hard-gate mode replaces activity's
smooth margin feedback with a binary penalty, removing information about
how far a candidate lies from the activity threshold.

\begin{table*}[t]
  \centering
  \caption{\textbf{DNA cell-type specificity: \emph{Activity $+$ MaxGap}.} HepG2 activity paired with a specificity margin against K562 and SK-N-SH. Unlike \emph{Activity $+$ ATAC} the two objectives are in tension. The activity term uses a hard gate or a soft gate, while MaxGap retains its soft gate. With a hard gate, activity's own pass rate \emph{falls} and the designs drift toward uniform composition. Uniqueness, 3-mer $r$ and JASPAR $r$ are guardrails, not objectives, and are reported without ranking. $N=640$ per seed, mean $\pm$ s.d. over three seeds; time is ms per design.}
  \label{tab:dna-maxgap}
  \setlength{\tabcolsep}{3.6pt}
  \renewcommand{\arraystretch}{1.13}
  \resizebox{\textwidth}{!}{%
  \begin{tabular}{lIcIccIccJcIcIcIc}
    \dimosOuterRule
    \multirow{2}{*}{\textbf{Method}}
      & \multirow{2}{*}{\textbf{Joint SR (\%) $\uparrow$}}
      & \multicolumn{2}{cI}{\textbf{Activity $\uparrow$}}
      & \multicolumn{2}{cJ}{\textbf{MaxGap $\uparrow$}}
      & \multirow{2}{*}{\textbf{Uniq. (\%)}}
      & \multirow{2}{*}{\textbf{3-mer $r$}}
      & \multirow{2}{*}{\textbf{JASPAR $r$}}
      & \multirow{2}{*}{\textbf{Time (ms)}} \\
      & & \textbf{Median} & \textbf{Top-10}
      & \textbf{Median} & \textbf{Top-10}
      & & & & \\
    \dimosMajorRule
    \rowcolor{dimosTask!70}
    \multicolumn{10}{c}{\color{dimosNavy}\textit{Hard-Gate Mode}} \\
    \baseline{Pretrained}
      & \meanstd{0.00}{0.00}
      & \meanstd{0.18}{0.01}
      & \meanstd{2.84}{0.25}
      & \meanstd{-0.12}{0.00}
      & \meanstd{0.39}{0.30}
      & \meanstd{100.00}{0.00}
      & \meanstd{-0.04}{0.01}
      & \meanstd{0.29}{0.01}
      & \meanstd{8.48}{0.20} \\
    \baseline{Best-of-N}
      & \meanstd{1.20}{0.27}
      & \meanstd{1.16}{0.05}
      & \meanstd{4.49}{0.24}
      & \meanstd{0.44}{0.00}
      & \meanstd{2.37}{0.65}
      & \meanstd{100.00}{0.00}
      & \meanstd{0.40}{0.01}
      & \meanstd{0.74}{0.00}
      & \meanstd{273.38}{0.32} \\
    \baseline{TDS}
      & \meanstd{7.71}{4.23}
      & \firstresult{\meanstd{4.37}{0.10}}
      & \meanstd{5.23}{0.20}
      & \meanstd{-0.37}{0.20}
      & \meanstd{1.88}{0.45}
      & \meanstd{46.30}{1.62}
      & \meanstd{0.77}{0.06}
      & \meanstd{0.85}{0.04}
      & \meanstd{171.94}{0.70} \\
    \baseline{SMC}
      & \meanstd{5.89}{1.42}
      & \secondresult{\meanstd{4.35}{0.10}}
      & \meanstd{4.79}{0.13}
      & \meanstd{-0.44}{0.07}
      & \meanstd{2.11}{0.49}
      & \meanstd{47.34}{3.77}
      & \meanstd{0.79}{0.03}
      & \meanstd{0.89}{0.01}
      & \meanstd{75.43}{0.19} \\
    \baseline{FK-Steering}
      & \meanstd{1.72}{0.34}
      & \meanstd{0.59}{0.02}
      & \meanstd{4.61}{0.11}
      & \meanstd{0.19}{0.01}
      & \meanstd{3.02}{0.47}
      & \meanstd{100.00}{0.00}
      & \meanstd{0.18}{0.04}
      & \meanstd{0.52}{0.03}
      & \meanstd{69.84}{0.61} \\
    \baseline{SVDD}
      & \secondresult{\meanstd{35.16}{1.60}}
      & \meanstd{3.71}{0.06}
      & \secondresult{\meanstd{6.55}{0.09}}
      & \secondresult{\meanstd{1.62}{0.07}}
      & \secondresult{\meanstd{5.59}{0.35}}
      & \meanstd{100.00}{0.00}
      & \meanstd{0.46}{0.01}
      & \meanstd{0.74}{0.01}
      & \meanstd{280.22}{0.44} \\
    \dimosMinorRule
    \rowcolor{dimosOurs!70}
    \textbf{\methodname\ (Ours)}
      & \firstresult{\meanstd{57.03}{0.92}}
      & \meanstd{4.34}{0.02}
      & \firstresult{\meanstd{7.54}{0.07}}
      & \firstresult{\meanstd{3.11}{0.12}}
      & \firstresult{\meanstd{6.68}{0.29}}
      & \meanstd{100.00}{0.00}
      & \meanstd{0.26}{0.01}
      & \meanstd{0.74}{0.01}
      & \meanstd{275.35}{0.07} \\
    \dimosMajorRule
    \rowcolor{dimosTask!70}
    \multicolumn{10}{c}{\color{dimosNavy}\textit{Soft-Gate Mode}} \\
    \baseline{Pretrained}
      & \meanstd{0.00}{0.00}
      & \meanstd{0.18}{0.01}
      & \meanstd{2.84}{0.25}
      & \meanstd{-0.12}{0.00}
      & \meanstd{0.39}{0.30}
      & \meanstd{100.00}{0.00}
      & \meanstd{-0.04}{0.01}
      & \meanstd{0.29}{0.01}
      & \meanstd{8.48}{0.20} \\
    \baseline{Best-of-N}
      & \meanstd{1.25}{0.22}
      & \meanstd{2.40}{0.03}
      & \meanstd{4.49}{0.23}
      & \meanstd{0.07}{0.04}
      & \meanstd{2.37}{0.65}
      & \meanstd{100.00}{0.00}
      & \meanstd{0.49}{0.03}
      & \meanstd{0.78}{0.02}
      & \meanstd{274.49}{1.16} \\
    \baseline{TDS}
      & \meanstd{9.79}{1.96}
      & \meanstd{2.95}{0.03}
      & \meanstd{5.40}{0.67}
      & \meanstd{0.21}{0.24}
      & \meanstd{3.32}{0.33}
      & \meanstd{46.98}{0.70}
      & \meanstd{0.28}{0.10}
      & \meanstd{0.42}{0.12}
      & \meanstd{167.61}{1.22} \\
    \baseline{SMC}
      & \meanstd{11.09}{5.18}
      & \meanstd{2.71}{0.55}
      & \meanstd{5.34}{0.29}
      & \meanstd{0.21}{0.12}
      & \meanstd{3.08}{0.41}
      & \meanstd{48.80}{0.48}
      & \meanstd{0.34}{0.24}
      & \meanstd{0.53}{0.19}
      & \meanstd{75.09}{0.22} \\
    \baseline{FK-Steering}
      & \meanstd{1.72}{0.46}
      & \meanstd{1.63}{0.11}
      & \meanstd{4.62}{0.14}
      & \meanstd{-0.01}{0.02}
      & \meanstd{2.96}{0.49}
      & \meanstd{100.00}{0.00}
      & \meanstd{0.39}{0.04}
      & \meanstd{0.71}{0.03}
      & \meanstd{69.80}{0.26} \\
    \baseline{SVDD}
      & \secondresult{\meanstd{44.06}{2.13}}
      & \secondresult{\meanstd{4.05}{0.04}}
      & \secondresult{\meanstd{7.10}{0.13}}
      & \secondresult{\meanstd{1.56}{0.07}}
      & \secondresult{\meanstd{4.83}{0.39}}
      & \meanstd{100.00}{0.00}
      & \meanstd{0.60}{0.00}
      & \meanstd{0.84}{0.00}
      & \meanstd{280.49}{0.39} \\
    \dimosMinorRule
    \rowcolor{dimosOurs!70}
    \textbf{\methodname\ (Ours)}
      & \firstresult{\meanstd{72.60}{0.29}}
      & \firstresult{\meanstd{5.06}{0.05}}
      & \firstresult{\meanstd{7.69}{0.11}}
      & \firstresult{\meanstd{2.26}{0.10}}
      & \firstresult{\meanstd{5.50}{0.25}}
      & \meanstd{100.00}{0.00}
      & \meanstd{0.73}{0.01}
      & \meanstd{0.87}{0.00}
      & \meanstd{275.73}{0.42} \\
    \dimosOuterRule
  \end{tabular}%
  }
\end{table*}

\subsection{Protein Hydrophilicity (GRAVY)}
\label{app:protein-gravy}

Tab.~\ref{tab:protein-gravy} repeats the protein comparison with the secondary objective
replaced by hydrophilicity, GRAVY $< -0.70$, using the Kyte--Doolittle
hydropathy scale~\citep{kyte1982simple} under the same two gate modes.
Two features of this task are worth isolating.

First, \texttt{TDS} reaches $0.00\%$ joint success in both gate modes,
despite clearing the binding threshold on approximately 52--54\% of its
designs. No output satisfies the hydrophilicity requirement.

Second, the wall-clock axis runs the other way here. \methodname\ costs about
$80$\,ms per design in both gate modes, \emph{below} every published baseline budget, so the
budget-matched partners are \texttt{Best-of-N} at $K{=}8$ and \texttt{SVDD} at $M{=}4$
rather than the larger settings used elsewhere. At that budget \methodname\ reaches
27.39\% against 18.73\% and
19.12\%, and it also exceeds the strongest baseline at \emph{any}
budget we measured. Both GRAVY configurations use $M{=}2$ and $K{=}2$.

\begin{table*}[t]
  \centering
  \caption{\textbf{Protein binder design, $\Delta\Delta G > 0$ and GRAVY $< -0.70$.} Same gate modes and protocol as Table~\ref{tab:protein-main}. \texttt{Best-of-N} and \texttt{SVDD} are shown at the budget closest to \methodname's wall clock, which in the hard-gate block is \emph{below} their published budgets because \methodname\ runs at $79$\,ms per design there. Time is milliseconds per design.}
  \label{tab:protein-gravy}
  \setlength{\tabcolsep}{3.6pt}
  \renewcommand{\arraystretch}{1.13}
  \resizebox{\textwidth}{!}{%
  \begin{tabular}{lIcIccIccJcIcIcIc}
    \dimosOuterRule
    \multirow{2}{*}{\textbf{Method}}
      & \multirow{2}{*}{\textbf{Joint SR (\%) $\uparrow$}}
      & \multicolumn{2}{cI}{\textbf{$\Delta\Delta G$ $\uparrow$}}
      & \multicolumn{2}{cJ}{\textbf{GRAVY $\downarrow$}}
      & \multirow{2}{*}{\textbf{Uniq. (\%)}}
      & \multirow{2}{*}{\textbf{scRMSD (\%)}}
      & \multirow{2}{*}{\textbf{pLDDT (\%)}}
      & \multirow{2}{*}{\textbf{Time (ms)}} \\
      & & \textbf{Median} & \textbf{Top-10}
      & \textbf{Median} & \textbf{Top-10}
      & & & & \\
    \dimosMajorRule
    \rowcolor{dimosTask!70}
    \multicolumn{10}{c}{\color{dimosNavy}\textit{Hard-Gate Mode}} \\
    \baseline{Pretrained}
      & \meanstd{6.84}{0.21}
      & \meanstd{-0.45}{0.00}
      & \meanstd{0.28}{0.02}
      & \meanstd{-0.57}{0.00}
      & \meanstd{-0.77}{0.00}
      & \meanstd{99.67}{0.05}
      & \meanstd{93.55}{0.11}
      & \meanstd{87.98}{0.11}
      & \meanstd{16.90}{0.08} \\
    \baseline{Best-of-N}
      & \meanstd{18.73}{0.16}
      & \secondresult{\meanstd{-0.28}{0.00}}
      & \secondresult{\meanstd{0.44}{0.02}}
      & \meanstd{-0.75}{0.00}
      & \meanstd{-0.86}{0.00}
      & \meanstd{98.31}{0.09}
      & \meanstd{93.55}{0.05}
      & \meanstd{87.96}{0.14}
      & \meanstd{81.64}{0.09} \\
    \baseline{TDS}
      & \meanstd{0.00}{0.00}
      & \firstresult{\meanstd{0.18}{0.03}}
      & \meanstd{0.36}{0.04}
      & \meanstd{0.43}{0.01}
      & \meanstd{0.24}{0.01}
      & \meanstd{11.11}{0.27}
      & \meanstd{88.76}{0.35}
      & \meanstd{84.38}{0.00}
      & \meanstd{121.42}{0.75} \\
    \baseline{SMC}
      & \meanstd{16.93}{1.24}
      & \secondresult{\meanstd{-0.28}{0.01}}
      & \meanstd{0.40}{0.06}
      & \meanstd{-0.51}{0.00}
      & \meanstd{-0.75}{0.01}
      & \meanstd{22.59}{0.77}
      & \meanstd{94.92}{0.63}
      & \meanstd{86.65}{0.37}
      & \meanstd{80.37}{0.06} \\
    \baseline{FK-Steering}
      & \meanstd{12.89}{0.14}
      & \secondresult{\meanstd{-0.28}{0.00}}
      & \firstresult{\meanstd{0.51}{0.01}}
      & \meanstd{-0.64}{0.00}
      & \meanstd{-0.79}{0.01}
      & \meanstd{99.18}{0.11}
      & \meanstd{93.06}{0.16}
      & \meanstd{87.48}{0.11}
      & \meanstd{192.31}{0.03} \\
    \baseline{SVDD}
      & \secondresult{\meanstd{19.12}{0.27}}
      & \meanstd{-0.61}{0.03}
      & \meanstd{0.33}{0.01}
      & \secondresult{\meanstd{-0.95}{0.00}}
      & \secondresult{\meanstd{-1.00}{0.01}}
      & \meanstd{99.09}{0.09}
      & \meanstd{92.58}{0.09}
      & \meanstd{88.00}{0.27}
      & \meanstd{85.79}{0.02} \\
    \dimosMinorRule
    \rowcolor{dimosOurs!70}
    \textbf{\methodname\ (Ours)}
      & \firstresult{\meanstd{27.39}{0.12}}
      & \meanstd{-0.71}{0.02}
      & \meanstd{0.40}{0.02}
      & \firstresult{\meanstd{-1.08}{0.00}}
      & \firstresult{\meanstd{-1.05}{0.00}}
      & \meanstd{100.00}{0.00}
      & \meanstd{91.86}{0.23}
      & \meanstd{84.85}{0.13}
      & \meanstd{79.98}{0.19} \\
    \dimosMajorRule
    \rowcolor{dimosTask!70}
    \multicolumn{10}{c}{\color{dimosNavy}\textit{Soft-Gate Mode}} \\
    \baseline{Pretrained}
      & \meanstd{6.84}{0.21}
      & \meanstd{-0.45}{0.00}
      & \meanstd{0.28}{0.02}
      & \meanstd{-0.57}{0.00}
      & \meanstd{-0.77}{0.00}
      & \meanstd{99.67}{0.05}
      & \meanstd{93.55}{0.11}
      & \meanstd{87.98}{0.11}
      & \meanstd{16.73}{0.10} \\
    \baseline{Best-of-N}
      & \meanstd{21.83}{0.27}
      & \meanstd{-0.28}{0.00}
      & \meanstd{0.49}{0.01}
      & \meanstd{-0.76}{0.00}
      & \meanstd{-0.86}{0.00}
      & \meanstd{98.63}{0.09}
      & \meanstd{91.60}{0.11}
      & \meanstd{88.19}{0.16}
      & \meanstd{81.48}{0.08} \\
    \baseline{TDS}
      & \meanstd{0.00}{0.00}
      & \firstresult{\meanstd{0.12}{0.01}}
      & \meanstd{0.39}{0.05}
      & \meanstd{0.42}{0.01}
      & \meanstd{0.22}{0.00}
      & \meanstd{10.37}{0.21}
      & \meanstd{89.97}{0.83}
      & \meanstd{85.42}{0.00}
      & \meanstd{117.91}{0.01} \\
    \baseline{SMC}
      & \meanstd{10.13}{0.35}
      & \meanstd{-0.43}{0.02}
      & \meanstd{0.27}{0.04}
      & \meanstd{-0.60}{0.01}
      & \meanstd{-0.81}{0.00}
      & \meanstd{23.61}{0.91}
      & \meanstd{93.55}{0.23}
      & \meanstd{87.48}{0.52}
      & \meanstd{78.98}{0.17} \\
    \baseline{FK-Steering}
      & \meanstd{13.41}{0.09}
      & \meanstd{-0.41}{0.00}
      & \meanstd{0.35}{0.02}
      & \meanstd{-0.71}{0.01}
      & \meanstd{-0.86}{0.00}
      & \meanstd{99.26}{0.03}
      & \meanstd{92.25}{0.11}
      & \meanstd{88.06}{0.08}
      & \meanstd{192.13}{0.08} \\
    \baseline{SVDD}
      & \secondresult{\meanstd{36.22}{0.26}}
      & \meanstd{-0.15}{0.02}
      & \secondresult{\meanstd{0.54}{0.01}}
      & \secondresult{\meanstd{-0.86}{0.00}}
      & \secondresult{\meanstd{-0.96}{0.00}}
      & \meanstd{97.55}{0.08}
      & \meanstd{90.47}{0.03}
      & \meanstd{88.89}{0.03}
      & \meanstd{85.53}{0.03} \\
    \dimosMinorRule
    \rowcolor{dimosOurs!70}
    \textbf{\methodname\ (Ours)}
      & \firstresult{\meanstd{46.94}{0.05}}
      & \secondresult{\meanstd{-0.03}{0.01}}
      & \firstresult{\meanstd{0.84}{0.03}}
      & \firstresult{\meanstd{-0.98}{0.00}}
      & \firstresult{\meanstd{-1.02}{0.00}}
      & \meanstd{99.74}{0.05}
      & \meanstd{91.04}{0.13}
      & \meanstd{85.50}{0.06}
      & \meanstd{80.35}{0.47} \\
    \dimosOuterRule
  \end{tabular}%
  }
\end{table*}

\subsection{RNA Design}
\label{app:rna-main}

We additionally evaluate RNA design~\citep{lorenz2011viennarna,li2024derivative}
on translation paired with accessibility or folding stability.
The two gate modes differ in the translation term: it uses a hard gate
or a soft gate, while the secondary objective retains its soft gate.
Both modes use the same success criteria.

\paragraph{Translation and Accessibility.}

Tab.~\ref{tab:rna-main} is the RNA translation $+$ accessibility result, moved
here for space. Both objectives use separate reward models for guidance and
final evaluation: guidance uses a ConvGRU MRL reward model and a
constrained-MFE opening proxy, while the table
reports Optimus-5$'$ MRL and the exact partition-function opening energy. The
gap between the two reveals discrepancies between proxy and exact
structural rewards. In hard-gate mode, reward-guided baselines pass the
opening proxy criterion on 14--77\% of designs but the exact criterion on
only 1.6--4.7\%, while the corresponding rates for \methodname\ are
99.7\% and 85.2\%.

For \methodname, replacing the translation soft gate with a hard gate raises
joint SR from 78.39\% to 84.11\%. The held-out translation pass rate remains
high (99.74\% versus 98.96\%), while the exact accessibility pass rate rises
from 78.65\% to 85.16\%. This is consistent with directing more reward
feedback toward the remaining structural requirement.

\begin{table*}[t]
  \centering
  \caption{\textbf{RNA 5$'$UTR design with hard and soft gates.} Both blocks are the same \emph{translation $+$ accessibility} task with the same success criterion, and differ only in whether the translation term uses a hard gate (binary penalty) or a soft gate (smooth reward). $N{=}128$ designs per seed. \textbf{Both objectives are graded by held-out oracles}: guidance sees a ConvGRU MRL predictor and a constrained-MFE opening proxy, while the table reports Optimus-5$'$ MRL and the exact partition-function opening energy $-RT\log p_{\mathrm{unpaired}}$, a different physical quantity rather than a finer estimate of the proxy. Each baseline is shown at its measured time-closest budget rung; none runs below our budget. Mean $\pm$ s.d. over three seeds; time is milliseconds per design.}
  \label{tab:rna-main}
  \setlength{\tabcolsep}{3.6pt}
  \renewcommand{\arraystretch}{1.13}
  \resizebox{\textwidth}{!}{%
  \begin{tabular}{lIcIccIccJcIcIc}
    \dimosOuterRule
    \multirow{2}{*}{\textbf{Method}}
      & \multirow{2}{*}{\textbf{Joint SR (\%) $\uparrow$}}
      & \multicolumn{2}{cI}{\textbf{Opening $\Delta G$ $\downarrow$}}
      & \multicolumn{2}{cJ}{\textbf{ConvGRU-MRL $\uparrow$}}
      & \multirow{2}{*}{\textbf{Uniq. (\%)}}
      & \multirow{2}{*}{\textbf{Optimus-MRL $\uparrow$}}
      & \multirow{2}{*}{\textbf{Time (ms)}} \\
      & & \textbf{Median} & \textbf{Top-10}
      & \textbf{Median} & \textbf{Top-10}
      & & & \\
    \dimosMajorRule
    \rowcolor{dimosTask!70}
    \multicolumn{9}{c}{\color{dimosNavy}\textit{Hard-Gate Mode}} \\
    \baseline{Pretrained}
      & \meanstd{0.52}{0.45}
      & \meanstd{4.07}{0.17}
      & \meanstd{2.01}{0.40}
      & \meanstd{0.34}{0.08}
      & \meanstd{0.81}{0.03}
      & \meanstd{100.00}{0.00}
      & \meanstd{-0.11}{0.11}
      & \meanstd{10.69}{0.08} \\
    \baseline{Best-of-N}
      & \meanstd{1.04}{1.19}
      & \meanstd{2.04}{0.08}
      & \meanstd{0.73}{0.03}
      & \meanstd{0.89}{0.01}
      & \meanstd{0.88}{0.00}
      & \meanstd{100.00}{0.00}
      & \meanstd{0.55}{0.02}
      & \meanstd{186.22}{0.32} \\
    \baseline{FK-Steering}
      & \meanstd{1.82}{1.80}
      & \meanstd{2.25}{0.16}
      & \meanstd{0.77}{0.26}
      & \meanstd{0.81}{0.01}
      & \meanstd{0.89}{0.02}
      & \meanstd{100.00}{0.00}
      & \meanstd{0.36}{0.04}
      & \meanstd{725.52}{9.83} \\
    \baseline{SMC}
      & \secondresult{\meanstd{2.86}{4.30}}
      & \secondresult{\meanstd{0.93}{0.13}}
      & \meanstd{0.51}{0.10}
      & \secondresult{\meanstd{0.93}{0.06}}
      & \secondresult{\meanstd{1.01}{0.10}}
      & \meanstd{28.39}{1.19}
      & \secondresult{\meanstd{0.62}{0.04}}
      & \meanstd{202.26}{8.12} \\
    \baseline{SVDD}
      & \meanstd{1.82}{0.90}
      & \meanstd{1.01}{0.09}
      & \secondresult{\meanstd{0.44}{0.02}}
      & \meanstd{0.86}{0.01}
      & \meanstd{0.97}{0.02}
      & \meanstd{100.00}{0.00}
      & \meanstd{0.50}{0.02}
      & \meanstd{347.65}{3.13} \\
    \dimosMinorRule
    \rowcolor{dimosOurs!70}
    \textbf{\methodname\ (Ours)}
      & \firstresult{\meanstd{84.11}{3.69}}
      & \firstresult{\meanstd{0.13}{0.00}}
      & \firstresult{\meanstd{0.06}{0.01}}
      & \firstresult{\meanstd{0.98}{0.01}}
      & \firstresult{\meanstd{1.08}{0.01}}
      & \meanstd{100.00}{0.00}
      & \firstresult{\meanstd{0.74}{0.01}}
      & \meanstd{178.24}{1.39} \\
    \dimosMajorRule
    \rowcolor{dimosTask!70}
    \multicolumn{9}{c}{\color{dimosNavy}\textit{Soft-Gate Mode}} \\
    \baseline{Pretrained}
      & \meanstd{0.52}{0.45}
      & \meanstd{4.07}{0.17}
      & \meanstd{2.01}{0.40}
      & \meanstd{0.34}{0.08}
      & \meanstd{0.81}{0.03}
      & \meanstd{100.00}{0.00}
      & \meanstd{-0.11}{0.11}
      & \meanstd{11.25}{0.88} \\
    \baseline{Best-of-N}
      & \meanstd{1.30}{1.63}
      & \meanstd{1.43}{0.08}
      & \meanstd{0.65}{0.02}
      & \meanstd{0.82}{0.01}
      & \meanstd{0.88}{0.01}
      & \meanstd{100.00}{0.00}
      & \meanstd{0.46}{0.00}
      & \meanstd{189.46}{1.19} \\
    \baseline{FK-Steering}
      & \meanstd{0.26}{0.45}
      & \meanstd{2.34}{0.18}
      & \meanstd{0.73}{0.04}
      & \meanstd{0.77}{0.02}
      & \meanstd{0.92}{0.02}
      & \meanstd{100.00}{0.00}
      & \meanstd{0.35}{0.01}
      & \meanstd{726.31}{7.12} \\
    \baseline{SMC}
      & \meanstd{0.00}{0.00}
      & \secondresult{\meanstd{0.66}{0.25}}
      & \meanstd{0.52}{0.04}
      & \secondresult{\meanstd{0.88}{0.07}}
      & \meanstd{0.96}{0.01}
      & \meanstd{34.11}{3.16}
      & \meanstd{0.40}{0.16}
      & \meanstd{202.14}{4.76} \\
    \baseline{SVDD}
      & \secondresult{\meanstd{4.17}{1.19}}
      & \meanstd{0.91}{0.13}
      & \secondresult{\meanstd{0.34}{0.05}}
      & \firstresult{\meanstd{1.07}{0.01}}
      & \firstresult{\meanstd{1.21}{0.01}}
      & \meanstd{100.00}{0.00}
      & \secondresult{\meanstd{0.69}{0.02}}
      & \meanstd{293.86}{4.65} \\
    \dimosMinorRule
    \rowcolor{dimosOurs!70}
    \textbf{\methodname\ (Ours)}
      & \firstresult{\meanstd{78.39}{4.96}}
      & \firstresult{\meanstd{0.14}{0.01}}
      & \firstresult{\meanstd{0.08}{0.01}}
      & \firstresult{\meanstd{1.07}{0.00}}
      & \secondresult{\meanstd{1.13}{0.01}}
      & \meanstd{100.00}{0.00}
      & \firstresult{\meanstd{0.81}{0.01}}
      & \meanstd{172.14}{1.06} \\
    \dimosOuterRule
  \end{tabular}%
  }
\end{table*}

\paragraph{Translation and Folding Stability.}
\label{app:rna-mfe}

Tab.~\ref{tab:rna-mfe} repeats the RNA comparison with accessibility replaced by the 50-nt
minimum free energy. Only the translation objective
uses a held-out reward model: MRL uses a ConvGRU reward model for guidance
and Optimus-5$'$ for evaluation, while the folding energy has no cheaper proxy
and is therefore computed by the same ViennaRNA call on both sides.

This task benefits from graded translation feedback: \methodname's soft-gate guidance
achieves 54.17\% joint SR and 62.50\% held-out translation success,
compared with 36.72\% and 51.30\% under hard-gate guidance. Together with
the accessibility results, this shows that the effect of the gate depends
on which requirements remain difficult to satisfy.

\begin{table*}[t]
  \centering
  \caption{\textbf{RNA 5$'$UTR design, translation $+$ folding stability.} As in Table~\ref{tab:rna-main}, with accessibility replaced by the 50-nt minimum free energy. MRL is graded by the held-out Optimus-5$'$ predictor; MFE has no cheaper guidance proxy and is the one objective shared between guidance and grading. Mean $\pm$ s.d. over three seeds; time is milliseconds per design.}
  \label{tab:rna-mfe}
  \setlength{\tabcolsep}{3.6pt}
  \renewcommand{\arraystretch}{1.13}
  \resizebox{\textwidth}{!}{%
  \begin{tabular}{lIcIccIccJcIcIc}
    \dimosOuterRule
    \multirow{2}{*}{\textbf{Method}}
      & \multirow{2}{*}{\textbf{Joint SR (\%) $\uparrow$}}
      & \multicolumn{2}{cI}{\textbf{MFE $\downarrow$}}
      & \multicolumn{2}{cJ}{\textbf{ConvGRU-MRL $\uparrow$}}
      & \multirow{2}{*}{\textbf{Uniq. (\%)}}
      & \multirow{2}{*}{\textbf{Optimus-MRL $\uparrow$}}
      & \multirow{2}{*}{\textbf{Time (ms)}} \\
      & & \textbf{Median} & \textbf{Top-10}
      & \textbf{Median} & \textbf{Top-10}
      & & & \\
    \dimosMajorRule
    \rowcolor{dimosTask!70}
    \multicolumn{9}{c}{\color{dimosNavy}\textit{Hard-Gate Mode}} \\
    \baseline{Pretrained}
      & \meanstd{0.26}{0.45}
      & \meanstd{-8.57}{0.53}
      & \meanstd{-12.13}{1.69}
      & \meanstd{0.34}{0.08}
      & \meanstd{0.75}{0.02}
      & \meanstd{100.00}{0.00}
      & \meanstd{-0.11}{0.11}
      & \meanstd{10.73}{0.23} \\
    \baseline{Best-of-N}
      & \meanstd{0.78}{1.35}
      & \meanstd{-9.38}{0.23}
      & \meanstd{-12.72}{0.88}
      & \firstresult{\meanstd{0.85}{0.01}}
      & \meanstd{0.84}{0.01}
      & \meanstd{100.00}{0.00}
      & \firstresult{\meanstd{0.47}{0.02}}
      & \meanstd{78.81}{2.70} \\
    \baseline{FK-Steering}
      & \meanstd{2.08}{0.45}
      & \meanstd{-11.39}{0.18}
      & \meanstd{-14.29}{1.16}
      & \meanstd{0.72}{0.03}
      & \meanstd{0.84}{0.01}
      & \meanstd{100.00}{0.00}
      & \meanstd{0.32}{0.06}
      & \meanstd{140.91}{1.77} \\
    \baseline{SMC}
      & \secondresult{\meanstd{18.75}{7.53}}
      & \secondresult{\meanstd{-14.77}{0.89}}
      & \meanstd{-16.10}{0.62}
      & \firstresult{\meanstd{0.85}{0.01}}
      & \meanstd{0.85}{0.02}
      & \meanstd{21.09}{4.13}
      & \secondresult{\meanstd{0.46}{0.02}}
      & \meanstd{68.72}{0.93} \\
    \baseline{SVDD}
      & \meanstd{15.62}{2.07}
      & \meanstd{-14.09}{0.21}
      & \secondresult{\meanstd{-18.40}{0.74}}
      & \secondresult{\meanstd{0.83}{0.00}}
      & \secondresult{\meanstd{0.86}{0.01}}
      & \meanstd{100.00}{0.00}
      & \meanstd{0.44}{0.03}
      & \meanstd{74.88}{3.49} \\
    \dimosMinorRule
    \rowcolor{dimosOurs!70}
    \textbf{\methodname\ (Ours)}
      & \firstresult{\meanstd{36.72}{2.07}}
      & \firstresult{\meanstd{-18.36}{0.17}}
      & \firstresult{\meanstd{-20.33}{0.27}}
      & \meanstd{0.82}{0.01}
      & \firstresult{\meanstd{0.88}{0.03}}
      & \meanstd{100.00}{0.00}
      & \secondresult{\meanstd{0.46}{0.02}}
      & \meanstd{36.12}{4.43} \\
    \dimosMajorRule
    \rowcolor{dimosTask!70}
    \multicolumn{9}{c}{\color{dimosNavy}\textit{Soft-Gate Mode}} \\
    \baseline{Pretrained}
      & \meanstd{0.26}{0.45}
      & \meanstd{-8.57}{0.53}
      & \meanstd{-12.13}{1.69}
      & \meanstd{0.34}{0.08}
      & \meanstd{0.75}{0.02}
      & \meanstd{100.00}{0.00}
      & \meanstd{-0.11}{0.11}
      & \meanstd{10.92}{0.43} \\
    \baseline{Best-of-N}
      & \meanstd{1.30}{0.90}
      & \meanstd{-11.32}{0.43}
      & \meanstd{-14.43}{0.36}
      & \meanstd{0.77}{0.02}
      & \meanstd{0.77}{0.01}
      & \meanstd{100.00}{0.00}
      & \meanstd{0.41}{0.02}
      & \meanstd{79.33}{1.58} \\
    \baseline{FK-Steering}
      & \meanstd{4.43}{2.26}
      & \meanstd{-11.62}{0.33}
      & \meanstd{-15.35}{0.53}
      & \meanstd{0.74}{0.03}
      & \meanstd{0.83}{0.02}
      & \meanstd{100.00}{0.00}
      & \meanstd{0.36}{0.03}
      & \meanstd{140.44}{2.46} \\
    \baseline{SMC}
      & \secondresult{\meanstd{21.35}{19.09}}
      & \secondresult{\meanstd{-17.19}{0.60}}
      & \meanstd{-17.47}{0.90}
      & \meanstd{0.77}{0.05}
      & \meanstd{0.87}{0.02}
      & \meanstd{26.04}{2.39}
      & \meanstd{0.34}{0.14}
      & \meanstd{70.06}{3.43} \\
    \baseline{SVDD}
      & \meanstd{19.27}{1.63}
      & \meanstd{-14.21}{0.64}
      & \secondresult{\meanstd{-18.31}{0.45}}
      & \secondresult{\meanstd{0.88}{0.01}}
      & \secondresult{\meanstd{0.91}{0.02}}
      & \meanstd{100.00}{0.00}
      & \firstresult{\meanstd{0.51}{0.03}}
      & \meanstd{68.83}{1.27} \\
    \dimosMinorRule
    \rowcolor{dimosOurs!70}
    \textbf{\methodname\ (Ours)}
      & \firstresult{\meanstd{54.17}{0.45}}
      & \firstresult{\meanstd{-18.72}{0.14}}
      & \firstresult{\meanstd{-20.10}{0.42}}
      & \firstresult{\meanstd{0.92}{0.00}}
      & \firstresult{\meanstd{0.99}{0.01}}
      & \meanstd{100.00}{0.00}
      & \secondresult{\meanstd{0.50}{0.02}}
      & \meanstd{35.71}{3.53} \\
    \dimosOuterRule
  \end{tabular}%
  }
\end{table*}

% TODO: Add supplementary derivations, protocols, and results.

\section{Continuous-Score Rewards}
\label{app:soft-objective}

Sec.~\ref{sec:method-reward} builds the joint reward from per-property
thresholds. The local resampling rule of Sec.~\ref{sec:method-doob} uses
$\exp(R/\beta)$ and applies directly to continuous property scores that
continue to increase beyond the acceptance thresholds.
Tabs.~\ref{tab:soft-objective-dna}
and~\ref{tab:soft-objective-protein} repeat the DNA and protein tasks with
the gated reward of Eq.~(\ref{eq:reward}) replaced by a weighted sum of the
two property scores, standardized on frozen-prior samples. The task, the
seven configurations and the success criteria are unchanged, and nothing
was re-tuned.

\begin{table*}[t]
  \centering
  \caption{\textbf{DNA design guided by a continuous property score.} The same \emph{Activity $+$ ATAC} task and the same seven configurations as Tab.~\ref{tab:dna-main}, with the gated reward of Eq.~(\ref{eq:reward}) replaced by a weighted sum of the two property scores standardized on frozen-prior samples, $R_\lambda=\lambda z_{\text{act}}+(1-\lambda)z_{\text{ATAC}}$. No acceptance threshold enters the reward, and nothing was re-tuned. Joint SR is scored against the unchanged canonical criteria. $\lambda$ weights the first objective and is part of the reward, not a sampler setting: at one $\lambda$ every method optimizes the same $R_\lambda$. Because $z$ is in units of the prior's own standard deviation, $\lambda=0.5$ means equal weight in prior spread, not equal importance. Blocks are different objective functions and are not comparable to each other. Guardrails are reported without ranking. $N=640$ per seed, mean $\pm$ s.d. over three seeds.}
  \label{tab:soft-objective-dna}
  \setlength{\tabcolsep}{3.6pt}
  \renewcommand{\arraystretch}{1.13}
  \resizebox{\textwidth}{!}{%
  \begin{tabular}{lIcIccIccJcIcIcIc}
    \dimosOuterRule
    \multirow{2}{*}{\textbf{Method}}
      & \multirow{2}{*}{\textbf{Joint SR (\%) $\uparrow$}}
      & \multicolumn{2}{cI}{\textbf{Activity $\uparrow$}}
      & \multicolumn{2}{cJ}{\textbf{ATAC $\uparrow$}}
      & \multirow{2}{*}{\textbf{Uniq. (\%)}}
      & \multirow{2}{*}{\textbf{3-mer $r$}}
      & \multirow{2}{*}{\textbf{JASPAR $r$}}
      & \multirow{2}{*}{\textbf{Time (ms)}} \\
      & & \textbf{Median} & \textbf{Top-10} & \textbf{Median} & \textbf{Top-10} & & & & \\
    \dimosMajorRule
    \rowcolor{dimosTask!70}
    \multicolumn{10}{c}{\color{dimosNavy}\textit{$\lambda=0.25$: ATAC-leaning}} \\
    \baseline{Pretrained}
      & \meanstd{0.00}{0.00}
      & \meanstd{0.18}{0.01}
      & \meanstd{2.41}{0.39}
      & \meanstd{0.03}{0.00}
      & \meanstd{0.65}{0.06}
      & \meanstd{100.00}{0.00}
      & \meanstd{-0.04}{0.01}
      & \meanstd{0.29}{0.01}
      & \meanstd{8.13}{0.01} \\
    \baseline{Best-of-N}
      & \meanstd{1.77}{0.39}
      & \meanstd{1.51}{0.03}
      & \meanstd{5.24}{0.26}
      & \meanstd{0.41}{0.01}
      & \meanstd{0.89}{0.03}
      & \meanstd{100.00}{0.00}
      & \meanstd{0.78}{0.00}
      & \meanstd{0.91}{0.00}
      & \meanstd{202.12}{0.10} \\
    \baseline{TDS}
      & \meanstd{3.70}{1.67}
      & \meanstd{1.11}{0.11}
      & \meanstd{5.28}{0.13}
      & \meanstd{0.28}{0.05}
      & \meanstd{0.84}{0.06}
      & \meanstd{51.88}{1.38}
      & \meanstd{0.37}{0.06}
      & \meanstd{0.56}{0.09}
      & \meanstd{204.19}{0.26} \\
    \baseline{SMC}
      & \meanstd{6.25}{4.69}
      & \meanstd{1.35}{0.29}
      & \meanstd{4.79}{0.79}
      & \meanstd{0.29}{0.17}
      & \meanstd{0.92}{0.01}
      & \meanstd{51.98}{1.19}
      & \meanstd{0.39}{0.09}
      & \meanstd{0.67}{0.08}
      & \meanstd{90.12}{0.31} \\
    \baseline{FK-Steering}
      & \meanstd{6.25}{0.44}
      & \meanstd{1.47}{0.04}
      & \meanstd{5.83}{0.22}
      & \meanstd{0.38}{0.02}
      & \secondresult{\meanstd{0.95}{0.01}}
      & \meanstd{100.00}{0.00}
      & \meanstd{0.71}{0.01}
      & \meanstd{0.89}{0.00}
      & \meanstd{110.83}{0.59} \\
    \baseline{SVDD}
      & \secondresult{\meanstd{47.34}{2.67}}
      & \secondresult{\meanstd{4.12}{0.13}}
      & \secondresult{\meanstd{7.23}{0.13}}
      & \secondresult{\meanstd{0.90}{0.00}}
      & \firstresult{\meanstd{0.99}{0.00}}
      & \meanstd{100.00}{0.00}
      & \meanstd{0.68}{0.01}
      & \meanstd{0.91}{0.01}
      & \meanstd{203.62}{0.50} \\
    \dimosMinorRule
    \rowcolor{dimosOurs!70}
    \textbf{\methodname\ (Ours)}
      & \firstresult{\meanstd{89.90}{1.19}}
      & \firstresult{\meanstd{5.75}{0.05}}
      & \firstresult{\meanstd{7.52}{0.18}}
      & \firstresult{\meanstd{0.99}{0.00}}
      & \firstresult{\meanstd{0.99}{0.00}}
      & \meanstd{100.00}{0.00}
      & \meanstd{0.91}{0.01}
      & \meanstd{0.93}{0.00}
      & \meanstd{195.46}{0.74} \\
    \dimosMajorRule
    \rowcolor{dimosTask!70}
    \multicolumn{10}{c}{\color{dimosNavy}\textit{$\lambda=0.50$: equal weight}} \\
    \baseline{Pretrained}
      & \meanstd{0.00}{0.00}
      & \meanstd{0.18}{0.01}
      & \meanstd{2.41}{0.39}
      & \meanstd{0.03}{0.00}
      & \meanstd{0.65}{0.06}
      & \meanstd{100.00}{0.00}
      & \meanstd{-0.04}{0.01}
      & \meanstd{0.29}{0.01}
      & \meanstd{8.34}{0.10} \\
    \baseline{Best-of-N}
      & \meanstd{1.77}{0.39}
      & \meanstd{2.06}{0.01}
      & \meanstd{5.24}{0.26}
      & \meanstd{0.37}{0.01}
      & \meanstd{0.89}{0.03}
      & \meanstd{100.00}{0.00}
      & \meanstd{0.77}{0.01}
      & \meanstd{0.90}{0.00}
      & \meanstd{207.86}{0.22} \\
    \baseline{TDS}
      & \meanstd{6.61}{1.98}
      & \meanstd{1.05}{0.17}
      & \meanstd{5.82}{0.78}
      & \meanstd{0.24}{0.09}
      & \meanstd{0.89}{0.04}
      & \meanstd{53.12}{2.01}
      & \meanstd{0.57}{0.19}
      & \meanstd{0.66}{0.13}
      & \meanstd{209.51}{0.40} \\
    \baseline{SMC}
      & \meanstd{7.40}{3.94}
      & \meanstd{1.29}{0.38}
      & \meanstd{5.56}{0.37}
      & \meanstd{0.16}{0.03}
      & \meanstd{0.85}{0.07}
      & \meanstd{52.76}{0.70}
      & \meanstd{0.55}{0.04}
      & \meanstd{0.73}{0.04}
      & \meanstd{89.10}{0.24} \\
    \baseline{FK-Steering}
      & \meanstd{7.71}{0.53}
      & \meanstd{2.03}{0.05}
      & \meanstd{5.88}{0.10}
      & \meanstd{0.33}{0.01}
      & \meanstd{0.94}{0.01}
      & \meanstd{100.00}{0.00}
      & \meanstd{0.70}{0.02}
      & \meanstd{0.90}{0.00}
      & \meanstd{110.64}{0.39} \\
    \baseline{SVDD}
      & \secondresult{\meanstd{49.84}{1.33}}
      & \secondresult{\meanstd{4.63}{0.12}}
      & \secondresult{\meanstd{7.26}{0.04}}
      & \secondresult{\meanstd{0.82}{0.00}}
      & \secondresult{\meanstd{0.98}{0.00}}
      & \meanstd{100.00}{0.00}
      & \meanstd{0.76}{0.00}
      & \meanstd{0.93}{0.00}
      & \meanstd{203.46}{0.42} \\
    \dimosMinorRule
    \rowcolor{dimosOurs!70}
    \textbf{\methodname\ (Ours)}
      & \firstresult{\meanstd{92.03}{0.46}}
      & \firstresult{\meanstd{5.85}{0.03}}
      & \firstresult{\meanstd{7.40}{0.13}}
      & \firstresult{\meanstd{0.98}{0.00}}
      & \firstresult{\meanstd{0.99}{0.00}}
      & \meanstd{100.00}{0.00}
      & \meanstd{0.93}{0.01}
      & \meanstd{0.94}{0.00}
      & \meanstd{195.50}{0.76} \\
    \dimosMajorRule
    \rowcolor{dimosTask!70}
    \multicolumn{10}{c}{\color{dimosNavy}\textit{$\lambda=0.75$: Activity-leaning}} \\
    \baseline{Pretrained}
      & \meanstd{0.00}{0.00}
      & \meanstd{0.18}{0.01}
      & \meanstd{2.41}{0.39}
      & \meanstd{0.03}{0.00}
      & \meanstd{0.65}{0.06}
      & \meanstd{100.00}{0.00}
      & \meanstd{-0.04}{0.01}
      & \meanstd{0.29}{0.01}
      & \meanstd{8.28}{0.01} \\
    \baseline{Best-of-N}
      & \meanstd{1.77}{0.39}
      & \meanstd{2.38}{0.02}
      & \meanstd{5.24}{0.26}
      & \meanstd{0.25}{0.00}
      & \meanstd{0.89}{0.03}
      & \meanstd{100.00}{0.00}
      & \meanstd{0.68}{0.02}
      & \meanstd{0.85}{0.01}
      & \meanstd{202.70}{0.22} \\
    \baseline{TDS}
      & \meanstd{8.54}{4.41}
      & \meanstd{2.03}{0.76}
      & \meanstd{5.49}{0.14}
      & \meanstd{0.18}{0.07}
      & \meanstd{0.90}{0.06}
      & \meanstd{53.33}{0.70}
      & \meanstd{0.57}{0.21}
      & \meanstd{0.65}{0.11}
      & \meanstd{208.40}{0.17} \\
    \baseline{SMC}
      & \meanstd{5.52}{5.22}
      & \meanstd{1.13}{0.28}
      & \meanstd{5.42}{0.35}
      & \meanstd{0.13}{0.02}
      & \meanstd{0.87}{0.09}
      & \meanstd{52.50}{0.34}
      & \meanstd{0.39}{0.14}
      & \meanstd{0.61}{0.14}
      & \meanstd{88.47}{0.32} \\
    \baseline{FK-Steering}
      & \meanstd{6.61}{0.32}
      & \meanstd{2.10}{0.09}
      & \meanstd{5.94}{0.17}
      & \meanstd{0.20}{0.01}
      & \meanstd{0.94}{0.02}
      & \meanstd{100.00}{0.00}
      & \meanstd{0.62}{0.03}
      & \meanstd{0.85}{0.02}
      & \meanstd{110.55}{0.38} \\
    \baseline{SVDD}
      & \secondresult{\meanstd{38.70}{1.47}}
      & \secondresult{\meanstd{4.98}{0.02}}
      & \secondresult{\meanstd{7.17}{0.05}}
      & \secondresult{\meanstd{0.62}{0.03}}
      & \secondresult{\meanstd{0.98}{0.00}}
      & \meanstd{100.00}{0.00}
      & \meanstd{0.77}{0.00}
      & \meanstd{0.94}{0.00}
      & \meanstd{203.41}{0.40} \\
    \dimosMinorRule
    \rowcolor{dimosOurs!70}
    \textbf{\methodname\ (Ours)}
      & \firstresult{\meanstd{91.77}{0.52}}
      & \firstresult{\meanstd{5.89}{0.00}}
      & \firstresult{\meanstd{7.60}{0.11}}
      & \firstresult{\meanstd{0.97}{0.00}}
      & \firstresult{\meanstd{0.99}{0.00}}
      & \meanstd{100.00}{0.00}
      & \meanstd{0.95}{0.00}
      & \meanstd{0.94}{0.00}
      & \meanstd{195.48}{0.72} \\
    \dimosOuterRule
  \end{tabular}%
  }
\end{table*}

\begin{table*}[t]
  \centering
  \caption{\textbf{Protein design guided by a continuous property score.} The same \emph{$\Delta\Delta G$ $+$ SeqRec} task and the same seven configurations as Tab.~\ref{tab:protein-main}, with the gated reward replaced by $R_\lambda=\lambda z_{\Delta\Delta G}+(1-\lambda)z_{\text{SeqRec}}$ on frozen-prior statistics. No acceptance threshold enters the reward, and nothing was re-tuned. Joint SR is scored against the unchanged $\Delta\Delta G>0$ and SeqRec $>0.55$. $\lambda$ weights $\Delta\Delta G$ and is part of the reward, not a sampler setting. Because $z$ is in units of the prior's own standard deviation, $\lambda=0.5$ means equal weight in prior spread, not equal importance. Blocks are different objective functions and are not comparable to each other. Guardrails are reported without ranking. $N=128$ per backbone over 12 backbones, mean $\pm$ s.d. over three seeds.}
  \label{tab:soft-objective-protein}
  \setlength{\tabcolsep}{3.6pt}
  \renewcommand{\arraystretch}{1.13}
  \resizebox{\textwidth}{!}{%
  \begin{tabular}{lIcIccIccJcIcIcIc}
    \dimosOuterRule
    \multirow{2}{*}{\textbf{Method}}
      & \multirow{2}{*}{\textbf{Joint SR (\%) $\uparrow$}}
      & \multicolumn{2}{cI}{\textbf{$\Delta\Delta G$ $\uparrow$}}
      & \multicolumn{2}{cJ}{\textbf{SeqRec $\uparrow$}}
      & \multirow{2}{*}{\textbf{Uniq. (\%)}}
      & \multirow{2}{*}{\textbf{scRMSD (\%)}}
      & \multirow{2}{*}{\textbf{pLDDT (\%)}}
      & \multirow{2}{*}{\textbf{Time (ms)}} \\
      & & \textbf{Median} & \textbf{Top-10} & \textbf{Median} & \textbf{Top-10} & & & & \\
    \dimosMajorRule
    \rowcolor{dimosTask!70}
    \multicolumn{10}{c}{\color{dimosNavy}\textit{$\lambda=0.25$: SeqRec-leaning}} \\
    \baseline{Pretrained}
      & \meanstd{4.04}{0.30}
      & \meanstd{-0.45}{0.00}
      & \meanstd{0.49}{0.03}
      & \meanstd{0.44}{0.00}
      & \meanstd{0.48}{0.00}
      & \meanstd{99.67}{0.05}
      & \meanstd{93.55}{0.11}
      & \meanstd{87.98}{0.11}
      & \meanstd{17.36}{0.04} \\
    \baseline{Best-of-N}
      & \meanstd{15.30}{0.19}
      & \meanstd{0.03}{0.01}
      & \meanstd{0.80}{0.02}
      & \meanstd{0.50}{0.00}
      & \meanstd{0.52}{0.00}
      & \meanstd{99.15}{0.16}
      & \meanstd{93.36}{0.16}
      & \meanstd{86.46}{0.16}
      & \meanstd{160.54}{0.12} \\
    \baseline{TDS}
      & \meanstd{3.84}{0.00}
      & \meanstd{-0.47}{0.04}
      & \meanstd{0.61}{0.11}
      & \meanstd{0.45}{0.00}
      & \meanstd{0.47}{0.00}
      & \meanstd{21.70}{0.72}
      & \meanstd{95.25}{0.42}
      & \meanstd{89.34}{0.15}
      & \meanstd{126.66}{0.03} \\
    \baseline{SMC}
      & \meanstd{6.60}{0.43}
      & \meanstd{-0.40}{0.02}
      & \meanstd{0.42}{0.01}
      & \meanstd{0.45}{0.00}
      & \meanstd{0.48}{0.00}
      & \meanstd{23.65}{0.80}
      & \meanstd{93.47}{0.62}
      & \meanstd{86.98}{0.14}
      & \meanstd{84.37}{0.84} \\
    \baseline{FK-Steering}
      & \meanstd{4.38}{0.06}
      & \meanstd{-0.46}{0.00}
      & \meanstd{0.57}{0.02}
      & \meanstd{0.45}{0.00}
      & \meanstd{0.49}{0.00}
      & \meanstd{99.72}{0.08}
      & \meanstd{93.29}{0.19}
      & \meanstd{86.65}{0.11}
      & \meanstd{193.09}{0.01} \\
    \baseline{SVDD}
      & \secondresult{\meanstd{33.59}{0.33}}
      & \secondresult{\meanstd{0.32}{0.01}}
      & \secondresult{\meanstd{0.98}{0.01}}
      & \secondresult{\meanstd{0.54}{0.00}}
      & \secondresult{\meanstd{0.57}{0.00}}
      & \meanstd{97.29}{0.22}
      & \meanstd{95.51}{0.18}
      & \meanstd{84.24}{0.05}
      & \meanstd{167.18}{0.10} \\
    \dimosMinorRule
    \rowcolor{dimosOurs!70}
    \textbf{\methodname\ (Ours)}
      & \firstresult{\meanstd{57.60}{0.27}}
      & \firstresult{\meanstd{0.42}{0.00}}
      & \firstresult{\meanstd{1.16}{0.02}}
      & \firstresult{\meanstd{0.61}{0.00}}
      & \firstresult{\meanstd{0.66}{0.00}}
      & \meanstd{100.00}{0.00}
      & \meanstd{92.56}{0.03}
      & \meanstd{83.92}{0.05}
      & \meanstd{158.24}{0.02} \\
    \dimosMajorRule
    \rowcolor{dimosTask!70}
    \multicolumn{10}{c}{\color{dimosNavy}\textit{$\lambda=0.50$: equal weight}} \\
    \baseline{Pretrained}
      & \meanstd{4.04}{0.30}
      & \meanstd{-0.45}{0.00}
      & \meanstd{0.58}{0.02}
      & \meanstd{0.44}{0.00}
      & \meanstd{0.47}{0.00}
      & \meanstd{99.67}{0.05}
      & \meanstd{93.55}{0.11}
      & \meanstd{87.98}{0.11}
      & \meanstd{17.50}{0.02} \\
    \baseline{Best-of-N}
      & \meanstd{15.45}{0.11}
      & \meanstd{0.14}{0.01}
      & \meanstd{0.92}{0.01}
      & \meanstd{0.48}{0.00}
      & \meanstd{0.50}{0.00}
      & \meanstd{99.35}{0.16}
      & \meanstd{93.45}{0.06}
      & \meanstd{86.02}{0.11}
      & \meanstd{160.49}{0.15} \\
    \baseline{TDS}
      & \meanstd{4.36}{0.98}
      & \meanstd{-0.46}{0.06}
      & \meanstd{0.72}{0.06}
      & \meanstd{0.44}{0.00}
      & \meanstd{0.46}{0.00}
      & \meanstd{22.24}{0.20}
      & \meanstd{94.12}{0.35}
      & \meanstd{88.00}{0.55}
      & \meanstd{125.66}{0.16} \\
    \baseline{SMC}
      & \meanstd{5.95}{0.34}
      & \meanstd{-0.55}{0.06}
      & \meanstd{0.34}{0.03}
      & \meanstd{0.44}{0.01}
      & \meanstd{0.48}{0.00}
      & \meanstd{23.39}{0.93}
      & \meanstd{92.62}{0.85}
      & \meanstd{86.48}{0.46}
      & \meanstd{83.71}{0.59} \\
    \baseline{FK-Steering}
      & \meanstd{4.36}{0.16}
      & \meanstd{-0.43}{0.00}
      & \meanstd{0.63}{0.01}
      & \meanstd{0.45}{0.00}
      & \meanstd{0.48}{0.00}
      & \meanstd{99.39}{0.03}
      & \meanstd{92.90}{0.16}
      & \meanstd{87.22}{0.08}
      & \meanstd{193.28}{0.04} \\
    \baseline{SVDD}
      & \secondresult{\meanstd{31.58}{0.11}}
      & \secondresult{\meanstd{0.46}{0.00}}
      & \secondresult{\meanstd{1.15}{0.00}}
      & \secondresult{\meanstd{0.53}{0.00}}
      & \secondresult{\meanstd{0.55}{0.00}}
      & \meanstd{97.46}{0.05}
      & \meanstd{94.23}{0.17}
      & \meanstd{84.57}{0.05}
      & \meanstd{166.98}{0.10} \\
    \dimosMinorRule
    \rowcolor{dimosOurs!70}
    \textbf{\methodname\ (Ours)}
      & \firstresult{\meanstd{46.05}{0.41}}
      & \firstresult{\meanstd{0.58}{0.01}}
      & \firstresult{\meanstd{1.28}{0.00}}
      & \firstresult{\meanstd{0.57}{0.00}}
      & \firstresult{\meanstd{0.61}{0.00}}
      & \meanstd{99.93}{0.05}
      & \meanstd{93.29}{0.05}
      & \meanstd{84.74}{0.08}
      & \meanstd{162.25}{5.94} \\
    \dimosMajorRule
    \rowcolor{dimosTask!70}
    \multicolumn{10}{c}{\color{dimosNavy}\textit{$\lambda=0.75$: $\Delta\Delta G$-leaning}} \\
    \baseline{Pretrained}
      & \meanstd{4.04}{0.30}
      & \meanstd{-0.45}{0.00}
      & \meanstd{0.64}{0.02}
      & \meanstd{0.44}{0.00}
      & \meanstd{0.47}{0.00}
      & \meanstd{99.67}{0.05}
      & \meanstd{93.55}{0.11}
      & \meanstd{87.98}{0.11}
      & \meanstd{17.28}{0.06} \\
    \baseline{Best-of-N}
      & \meanstd{13.65}{0.08}
      & \meanstd{0.24}{0.00}
      & \meanstd{1.01}{0.00}
      & \meanstd{0.46}{0.00}
      & \meanstd{0.49}{0.00}
      & \meanstd{99.13}{0.13}
      & \meanstd{93.88}{0.18}
      & \meanstd{85.83}{0.16}
      & \meanstd{159.62}{0.13} \\
    \baseline{TDS}
      & \meanstd{5.88}{0.30}
      & \meanstd{-0.34}{0.01}
      & \meanstd{0.54}{0.03}
      & \meanstd{0.44}{0.00}
      & \meanstd{0.47}{0.00}
      & \meanstd{23.81}{0.86}
      & \meanstd{91.51}{0.39}
      & \meanstd{88.85}{0.06}
      & \meanstd{119.42}{0.17} \\
    \baseline{SMC}
      & \meanstd{2.99}{0.51}
      & \meanstd{-0.56}{0.01}
      & \meanstd{0.44}{0.04}
      & \meanstd{0.44}{0.00}
      & \meanstd{0.47}{0.01}
      & \meanstd{24.76}{1.25}
      & \meanstd{93.90}{0.13}
      & \meanstd{87.65}{0.34}
      & \meanstd{80.04}{0.30} \\
    \baseline{FK-Steering}
      & \meanstd{5.92}{0.00}
      & \meanstd{-0.36}{0.00}
      & \meanstd{0.66}{0.01}
      & \meanstd{0.45}{0.00}
      & \meanstd{0.47}{0.00}
      & \meanstd{99.22}{0.11}
      & \meanstd{93.29}{0.23}
      & \meanstd{87.54}{0.20}
      & \meanstd{192.22}{0.01} \\
    \baseline{SVDD}
      & \secondresult{\meanstd{16.91}{0.16}}
      & \secondresult{\meanstd{0.57}{0.01}}
      & \secondresult{\meanstd{1.31}{0.01}}
      & \secondresult{\meanstd{0.49}{0.00}}
      & \secondresult{\meanstd{0.52}{0.00}}
      & \meanstd{98.03}{0.20}
      & \meanstd{95.59}{0.21}
      & \meanstd{85.05}{0.17}
      & \meanstd{166.82}{0.10} \\
    \dimosMinorRule
    \rowcolor{dimosOurs!70}
    \textbf{\methodname\ (Ours)}
      & \firstresult{\meanstd{19.51}{0.27}}
      & \firstresult{\meanstd{0.76}{0.01}}
      & \firstresult{\meanstd{1.40}{0.01}}
      & \firstresult{\meanstd{0.50}{0.00}}
      & \firstresult{\meanstd{0.54}{0.00}}
      & \meanstd{100.00}{0.00}
      & \meanstd{93.90}{0.12}
      & \meanstd{85.26}{0.03}
      & \meanstd{158.06}{0.07} \\
    \dimosOuterRule
  \end{tabular}%
  }
\end{table*}

\end{document}